\documentclass[letterpaper]{article} 
\usepackage{aaai2026}  
\usepackage{times}  
\usepackage{helvet}  
\usepackage{courier}  
\usepackage[hyphens]{url}  
\usepackage{graphicx} 
\usepackage{natbib}  
\usepackage{caption} 
\usepackage{algorithm}
\usepackage{algpseudocode}
\usepackage{amsmath,amsthm,amssymb}
\usepackage{amsfonts}
\usepackage{tikz}
\usepackage{enumitem}
\usetikzlibrary{calc,decorations.pathmorphing}
\usepackage{booktabs}

\usepackage{multirow}

\newtheorem{proposition}{Proposition}

\usepackage{newfloat}
\usepackage{listings}
\DeclareCaptionStyle{ruled}{labelfont=normalfont,labelsep=colon,strut=off} 
\floatstyle{ruled}
\newfloat{listing}{tb}{lst}{}
\floatname{listing}{Listing}

\title{GONDOR to the Rescue: Satisficing Planning with Low Memory}
\author{
    Yonatan Vernik\textsuperscript{\rm 1},
    Alexander Tuisov\textsuperscript{\rm 2},
    Alexander Shleyfman\textsuperscript{\rm 1}
}
\affiliations{
    \textsuperscript{\rm 1}Computer Science Department, Bar-Ilan University\\
    \textsuperscript{\rm 2}Independent Researcher\\
    \{yonatanw55, queldelan\}@gmail.com,
   alexander.shleyfman@biu.ac.il
}

\begin{document}

\maketitle

\begin{abstract}
Greedy Best-First Search (GBFS) is the dominant approach for solving search problems where the goal can be estimated with a heuristic, such as planning, route finding, navigation, and pathfinding. This is especially true when the memory is tightly constrained, such as planning on edge devices. To alleviate that, we present GONDOR (Greedy Online Navigation with Dynamic Outpost-based Re-search), a memory-efficient extension of GBFS that allows search to continue under strict memory limits by periodically compressing the search tree while retaining a sparse set of anchor states, then upon reaching the goal reconstructs the path by re-searching between the sparse states. We analyze the algorithm and discuss several variants defined by different outpost selection policies. In addition, we explore using Bloom filters for compact duplicate detection in the closed list. Experiments across numeric planning domains and heuristic configurations show that GONDOR consistently improves coverage under low memory budgets compared to standard GBFS. We release the implementation of GONDOR and the Bloom-filter variant to facilitate further research on memory-efficient heuristic search.
\end{abstract}


\section{Introduction}

In the past two decades Greedy Best-First Search (GBFS) de facto became the state-of-the-art for obtaining fast solutions for satisficing planning, mostly due to strong empirical performance when guided by informative heuristics \cite{bonet-geffner:aij-2001,helmert:jair-2006,taitler-etal:aimag-2024}.  In this setting, GBFS maintains an open list ordered by a heuristics $h$ and repeatedly expands the node with the smallest value, generating successors that are inserted back into the open list while expanded nodes are moved to a closed list. 

This greedy selection strategy often leads to rapid progress toward a solution, but it also introduces significant memory demands because the algorithm stores all nodes. Since GBFS does not use path-cost information to prune the search, large plateaus of nodes with similar heuristic values may accumulate, causing the search frontier to grow substantially \cite{tatsuya-kishimoto:socs-2011}.

A further complication is that GBFS is highly sensitive to heuristic inaccuracies. If the heuristic misguides the search, the algorithm may devote extensive effort to regions of the state space that do not contribute to a solution, expanding many nodes that are ultimately irrelevant. This phenomenon has been documented in studies analyzing GBFS search behavior, which show that the algorithm may expand large sets of states depending on tie-breaking and heuristic structure, even when those states do not lie on promising paths \cite{hausner-etal:socs-2017,hausner-etal:ijcai-2018a}.

Because both expanded nodes and generated-but-unexpanded nodes must be retained in memory, this behavior can lead to substantial memory consumption. Thus, GBFS may therefore exhaust available memory even after making substantial progress toward a goal, particularly when the heuristic induces broad plateaus or misleading gradients that cause extensive exploration before a solution is reached \cite{hausner-etal:ijcai-2018b,kuroiwa-beck:icaps-2022}. This limitation motivates many extensions to GBFS, including diversification strategies~\cite{tatsuya-kishimoto:socs-2011} and multi-queue variants \cite{richter-westphal:jair-2010,xie-etal:icaps-2015}.

To alleviate this, we introduce GONDOR (Greedy Online Navigation with Dynamic Outpost-based Re-search), that adds a memory-bounding mechanism to GBFS. When memory limits are reached, the algorithm prunes the search tree, retaining only a sparse subset of states called outposts. GONDOR then resumes the search from these outposts, effectively managing memory while maintaining progress.

Once a goal is reached, the outposts along the discovered route define a sequence of beacons. In a second phase, a complete plan is reconstructed by solving smaller exact-state subproblems between consecutive beacons. We discuss outpost-selection policies and perform experiments with a naive Bernoulli policy, as well as with a Bloom Filter-based closed list to further reduce memory consumption, tested on both standard memory and low-memory settings.



Much of the work on memory-efficient heuristic search has focused on optimal planning. In contrast, we consider standard fully observable deterministic satisficing planning \cite{ghallab-et-al:2004}. The MA* algorithm  \cite{chakrabarti-etal:aij-1989,russell:ecai-1992} also frees memory by removing unpromising states, but does so incrementally by “de-growing” the search tree one node at a time. Frontier search \cite{korf-et-al:jacm-2005} discards entire portions of the search tree, similar in spirit to our approach; however it does so incrementally and reconstructs the solution using bidirectional search, whereas we we drop at once on memory pressure, and reconstruct using short re-searches with beacons.

\section{Background}
We consider search algorithms defined over a \textbf{state space}
\(
\mathcal{S}=\langle S, s_I, S_\star, \mathrm{succ} \rangle,
\)
where \(S\) is a non-necessarily finite set of states, \(s_I \in S\) is the initial state,
\(S_\star \subseteq S\) is the set of goal states, and
\(\mathrm{succ}: S \rightarrow 2^{S}\) is a successor function that maps each
state \(s \in S\) to the finite  set of its successor states.

A sequence of pairwise distinct states
\(
\langle s_0, s_1, \ldots, s_n\rangle
\)
is a path from \(s_0\) to \(s_n\) if
$s_i \in \mathrm{succ}(s_{i-1})$ for $i \in \{1,\ldots,n\}$ .
A path \(\langle s_0,\ldots,s_n\rangle\) is a $s$-plan if \(s_0=s\) and \(s_n \in S_\star\). 

\noindent\textbf{Heuristic} \(h: S \rightarrow \mathbb{R}_{\ge 0}\) is a  function. Typically, \(h(s)\) estimates the cost of a cheapest \(s\)-plan, but this is not a necessary condition in the satisficing setting. In this work, the heuristic is treated as an arbitrary black-box function that assigns a finite non-negative real value to each state.

\noindent\textbf{Greedy Best-First
Search (GBFS)} takes as input a state-space \(
\mathcal{S}\)  and returns an \(s_I\)-plan if one exists,
and \emph{unsolvable} otherwise. GBFS is guided by the assumption that states
with lower heuristic values are closer to a goal state. At each iteration,
the algorithm expands a generated state with minimum
heuristic value (taken from the \textsc{open} list), and terminates once a goal state is generated. To reduce memory consumption, all generated states are kept in the \textsc{closed} list, and no state is expanded twice. Due to greediness, GBFS provides no guarantee on the quality of the
resulting \(s_I\)-plan. 


\noindent\textbf{Bloom filter} is a compact probabilistic data structure with a constant memory use for approximate set membership. It supports insertions and membership queries, may yield false positives, and does not yield false negatives. For further info see \cite{10.1145/362686.362692, 10.1007/11841036_61}. In some of our experiments below we use Bloom filter to represent the closed list, reducing memory at the cost of occasionally pruning a previously unseen state.

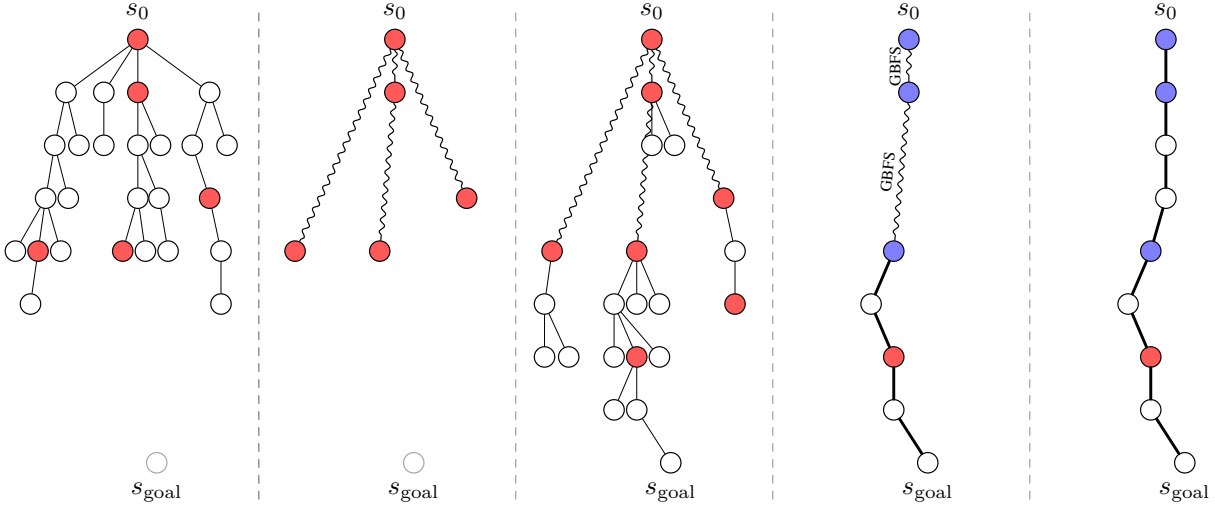
\begin{figure*}[t]
\centering
\begin{tikzpicture}[
    x=1cm,y=1cm,
    line cap=round,line join=round,
    edge/.style={black, line width=0.32pt},
    strongedge/.style={black, line width=1.1pt},
    squig/.style={
        black, line width=0.45pt,
        decorate,
        decoration={snake, amplitude=0.9pt, segment length=4.5pt}
    },
    nodebase/.style={
        circle, draw=black, line width=0.45pt,
        minimum size=2.7mm, inner sep=0pt
    },
    reg/.style={nodebase, fill=white},
    outpost/.style={nodebase, fill=red!65},
    beacon/.style={nodebase, fill=blue!50},
    goal/.style={nodebase, fill=white},
    ghostgoal/.style={nodebase, draw=black!35, fill=white},
]

%

\begin{scope}[shift={(0,0)}]

\node[outpost, label=above:$s_0$] (r1) at (0,0) {};

\node[reg]     (a1) at (-0.95,-0.70) {};
\node[outpost] (b1) at ( 0.00,-0.70) {};
\node[reg]     (c1) at ( 0.95,-0.70) {};
\node[reg]     (e1) at (-0.45,-0.70) {};
\draw[edge] (r1)--(a1); \draw[edge] (r1)--(b1);
\draw[edge] (r1)--(c1); \draw[edge] (r1)--(e1);

\node[reg] (a2)  at (-1.10,-1.40) {};
\node[reg] (a2s) at (-0.78,-1.40) {};
\draw[edge] (a1)--(a2); \draw[edge] (a1)--(a2s);

\node[reg] (b2)   at (0.00,-1.40) {};
\node[reg] (b2s1) at (0.30,-1.40) {};
\draw[edge] (b1)--(b2); \draw[edge] (b1)--(b2s1);

\node[reg] (c1s1) at (0.72,-1.40) {};
\node[reg] (c1s2) at (1.18,-1.40) {};
\draw[edge] (c1)--(c1s1); \draw[edge] (c1)--(c1s2);

\node[reg] (e2) at (-0.45,-1.40) {};
\draw[edge] (e1)--(e2);

\node[reg] (a3)  at (-1.22,-2.10) {};
\node[reg] (a3s) at (-0.92,-2.10) {};
\draw[edge] (a2)--(a3); \draw[edge] (a2)--(a3s);

\node[reg] (b3)  at (0.00,-2.10) {};
\node[reg] (b3s) at (0.28,-2.10) {};
\draw[edge] (b2)--(b3); \draw[edge] (b2)--(b3s);

\node[outpost] (c2) at (0.95,-2.10) {};
\draw[edge] (c1s1)--(c2);

\node[outpost] (a4)   at (-1.32,-2.80) {};
\node[reg]     (a4s1) at (-1.02,-2.80) {};
\node[reg]     (a4s2) at (-1.62,-2.80) {};
\draw[edge] (a3)--(a4); \draw[edge] (a3)--(a4s1); \draw[edge] (a3)--(a4s2);

\node[outpost] (b4)   at (-0.20,-2.80) {};
\node[reg]     (b4s1) at ( 0.10,-2.80) {};
\node[reg]     (b4s2) at ( 0.40,-2.80) {};
\draw[edge] (b3)--(b4); \draw[edge] (b3)--(b4s1); \draw[edge] (b3s)--(b4s2);

\node[reg] (c3) at (1.10,-2.80) {};
\draw[edge] (c2)--(c3);

\node[reg] (a5) at (-1.42,-3.50) {};
\draw[edge] (a4)--(a5);

\node[reg] (c4) at (1.10,-3.50) {};
\draw[edge] (c3)--(c4);

\node[ghostgoal, label=below:$s_{\mathrm{goal}}$] (g1) at (0.25,-5.60) {};

\draw[black!45, dashed, line width=0.35pt] (1.6,0.35) -- (1.6,-6.10);

\end{scope}

\begin{scope}[shift={(3.40,0)}]

\node[outpost, label=above:$s_0$] (r2) at (0,0) {};
\node[outpost] (b12) at ( 0.00,-0.70) {};
\node[outpost] (a42) at (-1.32,-2.80) {};
\node[outpost] (b42) at (-0.20,-2.80) {};
\node[outpost] (c22) at ( 0.95,-2.10) {};

\draw[squig] (r2)--(b12);
\draw[squig] (r2)--(a42);
\draw[squig] (b12)--(b42);
\draw[squig] (r2)--(c22);

\node[ghostgoal, label=below:$s_{\mathrm{goal}}$] (g2) at (0.25,-5.60) {};

\draw[black!45, dashed, line width=0.35pt] (1.6,0.35) -- (1.6,-6.10);

\end{scope}

\draw[black!45, dashed, line width=0.35pt] (1.6,0.35) -- (1.6,-6.10);

\begin{scope}[shift={(6.80,0)}]

\node[outpost, label=above:$s_0$] (r3) at (0,0) {};
\node[outpost] (b13) at ( 0.00,-0.70) {};
\node[outpost] (a43) at (-1.32,-2.80) {};
\node[outpost] (b43) at (-0.20,-2.80) {};
\node[outpost] (c23) at ( 0.95,-2.10) {};

\draw[squig] (r3)--(b13);
\draw[squig] (r3)--(a43);
\draw[squig] (b13)--(b43);
\draw[squig] (r3)--(c23);

\node[goal, label=below:$s_{\mathrm{goal}}$] (g3) at (0.25,-5.60) {};

\node[reg] (b23)   at ( 0.00,-1.40) {};
\node[reg] (b2s13) at ( 0.30,-1.40) {};
\draw[edge] (b13)--(b23); \draw[edge] (b13)--(b2s13);

\node[reg] (a53) at (-1.42,-3.50) {};  
\draw[edge] (a43)--(a53);
\node[reg] (a63) at (-1.42,-4.20) {};  
\node[reg] (a73) at (-1.10,-4.20) {};
\draw[edge] (a53)--(a63); \draw[edge] (a53)--(a73);

\node[reg]     (c33) at (1.10,-2.80) {};  
\node[outpost] (c43) at (1.10,-3.50) {};  
\draw[edge] (c23)--(c33); \draw[edge] (c33)--(c43);

\node[reg] (b53) at (-0.50,-3.50) {};  
\node[reg] (b63) at (-0.20,-3.50) {};
\node[reg] (b73) at ( 0.10,-3.50) {};
\draw[edge] (b43)--(b53); \draw[edge] (b43)--(b63); \draw[edge] (b43)--(b73);

\node[reg]     (b83)  at (-0.50,-4.20) {};  
\node[outpost] (b93)  at (-0.20,-4.20) {};
\node[reg]     (b103) at ( 0.10,-4.20) {};
\draw[edge] (b53)--(b83); \draw[edge] (b53)--(b93); \draw[edge] (b53)--(b103);

\node[reg] (b113) at (-0.50,-4.90) {};  
\node[reg] (b123) at (-0.20,-4.90) {};
\draw[edge] (b93)--(b113); \draw[edge] (b93)--(b123);

\draw[edge] (b123)--(g3);  

\draw[black!45, dashed, line width=0.35pt] (1.6,0.35) -- (1.6,-6.10);

\end{scope}

\begin{scope}[shift={(10.20,0)}]

\node[beacon, label=above:$s_0$] (r4) at (0,0) {};
\node[beacon] (b14) at ( 0.00,-0.70) {};
\node[beacon] (b44) at (-0.20,-2.80) {};

\draw[squig] (r4)--(b14);
\draw[squig] (b14)--(b44);

\node[goal, label=below:$s_{\mathrm{goal}}$] (g4) at (0.25,-5.60) {};

\node[reg]     (b54)  at (-0.50,-3.50) {};
\node[outpost] (b94)  at (-0.20,-4.20) {};
\node[reg]     (b124) at (-0.20,-4.90) {};
\draw[strongedge] (b44)--(b54);
\draw[strongedge] (b54)--(b94);
\draw[strongedge] (b94)--(b124);
\draw[strongedge] (b124)--(g4);

\node[font=\tiny, rotate=90] at ($(r4)!0.5!(b14)+(-0.18,0)$) {GBFS};
\node[font=\tiny, rotate=84] at ($(b14)!0.5!(b44)+(-0.18,0)$) {GBFS};

\draw[black!45, dashed, line width=0.35pt] (1.6,0.35) -- (1.6,-6.10);

\end{scope}

\begin{scope}[shift={(13.60,0)}]

\node[beacon, label=above:$s_0$] (r5)   at ( 0.00, 0.00) {};
\node[beacon]  (b15)  at ( 0.00,-0.70) {};
\node[reg]     (b25)  at ( 0.00,-1.40) {};
\node[reg]     (b35)  at ( 0.00,-2.10) {};
\node[beacon]  (b45)  at (-0.20,-2.80) {};
\node[reg]     (b55)  at (-0.50,-3.50) {};
\node[outpost] (b95)  at (-0.20,-4.20) {};
\node[reg]     (b125) at (-0.20,-4.90) {};
\node[goal, label=below:$s_{\mathrm{goal}}$] (g5) at (0.25,-5.60) {};

\draw[strongedge] (r5)  -- (b15);
\draw[strongedge] (b15) -- (b25);
\draw[strongedge] (b25) -- (b35);
\draw[strongedge] (b35) -- (b45);
\draw[strongedge] (b45) -- (b55);
\draw[strongedge] (b55) -- (b95);
\draw[strongedge] (b95) -- (b125);
\draw[strongedge] (b125)-- (g5);

\end{scope}
 
\end{tikzpicture}
\caption{Illustration of GONDOR. Our method starts by searching with GBFS expansion order (left). Some search nodes, always including $s_0$, are designated as outposts (in red). Once memory runs out, all the nodes are cleared except for the outposts, connected by $outpost\_parent$ pointers (left middle). Search continues from there (middle) until a goal is reached, with repeated clears every time the memory limit is reached. Once a goal is reached, we trace along the path until $parent=\bot$, then trace along $outpost\_parent$ to the start (right middle). Outposts along the missing segments are promoted to beacons (blue) and the rest of the tree is dropped. Finally, the missing path segments are determined by re-searching with GBFS and the results are concatenated, along with the plan suffix before, to create the final path (right).}
\label{fig:gondor-overview}
\end{figure*}

\section{Search with Limited Memory}

We introduce two intermediate observations that motivate the design of
GONDOR. These methods illustrate different ways of adapting GBFS to operate under strict memory limits. Consider the following algorithms. 

\noindent\textbf{GBFS-U (Unlimited)}  is a variant of GBFS that does not terminate when the
memory bound is exceeded. Instead, after each node expansion, if memory usage
surpasses a predefined threshold, the algorithm repeatedly removes the oldest
node from the search tree (and associated data structures) until the memory
usage falls below the threshold. The search then continues from the remaining
frontier.

Because nodes may be removed during search, parent pointers required for
solution reconstruction may be lost. Thus, GBFS-U may only
verify the existence of a solution, but it does
not necessarily preserve a complete plan.

\noindent\textbf{GBFS-R (Retrace)} extends GBFS-U by exploiting the
observation that when a goal state is reached, a suffix of the plan remains intact from a surviving ancestor to the goal. The algorithm iteratively reconstructs the full plan by repeatedly running
GBFS-U: each run searches from the initial state to the
earliest surviving ancestor of the previously recovered suffix. The recovered suffixes are concatenated until a complete plan is obtained.

This approach preserves the reduced memory footprint of
GBFS-U while enabling plan reconstruction. However, it may
require multiple re-searches from the initial state, which can substantially
increase runtime.

\noindent\textbf{GONDOR} Algorithm~\ref{alg:gondor} modifies the retracing procedure
by replacing repeated searches from the initial state with searches over shorter, non-overlapping segments. During the initial search, a subset of generated nodes is designated as \emph{outposts}. Each node additionally stores a pointer to its nearest ancestral outpost; the initial state is always an outpost.

Search proceeds similarly to GBFS until memory usage exceeds a threshold $L$
at which point a cleanup procedure \textsc{RemoveNonOutposts} is triggered.
During cleanup, all non-outpost nodes are discarded (this includes the search tree, the \textsc{open}  list,  and the  and the \textsc{closed} list). The search then resumes from the remaining outposts. Ordinary parent
pointers of outposts are reset, while outpost-parent pointers are preserved.
This effectively compresses the explored search history into a sparse chain of
anchoring states.

When a goal is found, GONDOR first extracts the remaining suffix of the plan using standard parent pointers.\footnote{This suffix always terminates
at an outpost, since after each cleanup only outposts remain in the \textsc{open} list.} The algorithm then follows outpost-parent pointers to obtain an ordered
sequence of outposts along the discovered route, referred to as
\emph{beacons}. The search tree is subsequently cleared, and a second phase reconstructs the full plan by re-solving the subproblems between consecutive beacons and concatenating the resulting segments with the retained suffix.

GONDOR includes a function hyper-parameter \textit{IsOutpost}, which determines whether a generated node is marked as an outpost.  Here, we use a simple Bernoulli policy in which each node independently becomes an outpost
with probability \(p\). Smaller values of \(p\) increase the effective memory available to the search, but may increase the cost of re-expansion after cleanup and the time required for plan reconstruction.
Empirically, the reconstruction time  for $p=0.01$ was typically small relative to the
initial search.

\paragraph{Theoretical Properties} Note that GBFS is complete on finite search spaces and does not have such guarantees on infinite ones. Since no algorithm can solve a plan existence problem within constant memory,  we analyze GONDOR under a setting where the memory available for each iteration is bounded by a constant \(L\), where $L$ is measured in the number of newly generated nodes.

Consider the stochastic outpost policy $\textsc{IsOutpost} \sim \mathrm{Ber}(p)$.
When $p=0$ and there is no threshold $L$, the algorithm reduces to the execution of GBFS. In this setting, assuming the same $h$ and tie breaking the algorithm returns the same solution as GBFS.

For $p>0$, additional memory is allocated to storing outposts, reducing the effective memory available to the main search. However, after each cleanup, the search restarts from a frontier containing the initial state and a subset of previously explored states, preserving soundness.

If GBFS with unbounded memory follows a solution path, then GONDOR expands progressively longer prefixes of that path. Since each generated node is selected as an outpost with probability $p>0$, the probability that none of the prefix states is retained decreases exponentially with its length. Thus, progress to the goal is preserved across cleanups. When later states have lower heuristic values, cleanup primarily removes less promising regions.

After a goal is found, only the beacons and final suffix are retained. Reconstruction solves the segments between consecutive beacons in reverse order. Each segment corresponds to a region previously explored under at least the same memory conditions, so no more states are expanded than in the original search. Hence, reconstruction proceeds without additional memory overflows under the same heuristic and tie-breaking assumptions.

\begin{proposition}
Given a bound \(L\) on the number of newly generated nodes per cleanup phase and
\(p>0\), \textsc{GONDOR} is almost surely complete on finite state spaces.
\end{proposition}

\begin{proof}
The algorithm terminates when a goal is found or when the \textsc{open} list is empty. Since the initial state is always an outpost, an empty \textsc{open} list implies that all reachable states were explored and the instance is unsolvable.

Between two cleanup phases, at most \(L\) nodes are generated, and each is marked as an outpost independently with probability \(p>0\). Hence, the probability of adding at least one new outpost in a phase is strictly positive. If no outpost is added, the same phase repeats. Therefore, with probability~1, after finitely many phases at least one additional state is added to the outpost set.

Because the state space is finite, eventually either a goal is found or all reachable states become outposts. In the latter case, no nodes are deleted and the remaining search exhausts the \textsc{open} list, reporting failure. Thus, \textsc{GONDOR} is almost surely complete.
\end{proof}

\begin{algorithm}[!t]
\caption{GONDOR}
\label{alg:gondor}
\begin{algorithmic}[1]
\State \textbf{Input:} initial state $s_0$, goal test $S_{\ast}$, heuristic $h$, successor function $\textsc{Succ}$, memory limit $L$
\State \textbf{Input:} outpost policy $\textsc{IsOutpost}$
\State \textbf{Output:} plan $\pi$ or failure
\State $r \leftarrow \textsc{NewNode}(s_0,\bot)$
\State $r.po \leftarrow \bot$; $r.is\_outpost \leftarrow True$ // po is parent\_outpost
\State $\textsc{open} \leftarrow \{r\}$; $\textsc{closed} \leftarrow \{s_0\}$
\While {$\textsc{open} \neq \emptyset$}
    \State $n \leftarrow$ pop node from $\textsc{open}$ according to $h$
    \If{$n.state \in S_{\ast}$}
        \State \textbf{return} \textsc{ReconstructViaBeacons}$(n, h)$
    \EndIf
        \For{each $s' \in \textsc{Succ}(n.state)$}
        \If{$s' \in \textsc{closed}$}
            \State \textbf{continue}
        \EndIf
        \State $m \leftarrow \textsc{NewNode}(s', n)$
        \State $m.po \leftarrow n~~\mathit{\textbf{if}}~~n.is\_outpost~~\mathit{\textbf{else}}~~n.po$
        
        \State{$m.is\_outpost \leftarrow \textsc{IsOutpost}(m)$}
        \State $\textsc{open} \leftarrow \textsc{open} \cup \{m\}$
        \State $\textsc{closed} \leftarrow \textsc{closed} \cup \{m.state\} $
    \EndFor

    \If{memory usage exceeds $L$}
        \State $\textsc{open} \leftarrow \textsc{RemoveNonOutposts}(\textsc{open})$
        \State $\textsc{closed} \leftarrow \textsc{open}.state$
    \EndIf
\EndWhile
\State \textbf{return} failure
\end{algorithmic}
\end{algorithm}
\section{Experimental results}

The experiments were conducted on a 13th Gen Intel®
Core™ i9-13900 processor running at 2.00 GHz, supported
by a 64-bit operating system and 32 GB of RAM. To evaluate our search algorithm on a diverse set of heuristics and domains, we follow \citet{esg_planning:corr-2025} to generate 10 domain-specific heuristics\footnote{The input to the LLM includes the domain and transition rules, but not the goal or the initial state which are different per instance} using GPT-5.1 with high reasoning effort and the default \textit{temperature}=1.0 and \textit{top\_p}=1.0. We have evaluated our approach on the domains from the Numeric International Planning Competition (IPC) 2023 \citep{taitler-etal:aimag-2024} as well as on the two domains introduced in \citet{esg_planning:corr-2025}; Twin-Primes and Pacman.

We ran each heuristic on each instance with a time limit of 10 minutes and a RAM limit of 8 GB (or 512 MB for low-memory experiments). We searched with the following settings: GBFS, GBFS$_{B\!F}$, GONDOR, and GONDOR$_{B\!F}$. All GONDOR settings use $p=0.01$ as described above, while the Bloom Filter settings used k=23 hash functions and 400MiB of RAM (fpr=$10^{-7}$ for set membership check at $10^8$ nodes)  for 8 GB limit, and k=20 and 3.5MiB RAM (fpr=$10^{-6}$ for set membership check at $10^6$ nodes) for 512 MB limit setting \cite{10.1007/11841036_61}.

Table \ref{tab:coverage} shows the coverage results. On the best heuristics all suggested methods equal or beat the baseline on all domains in both memory settings. For 8GB, gains are primarily in Counters (+4) and Drone (+1), while for 512 MB gains are more spread out including (+7) on TPP. When considering all heuristics, the suggested methods no longer dominate per-domain and every method has some domain it is best on, but GONDOR$_{B\!F}$ wins on most of the domains and overall.

As shown in Figure~\ref{fig:runtime}, GONDOR is slightly slower, though the gap decreases with increasing problem size. While GBFS frequently exhausts available memory after roughly 30--200 seconds (10--40 under tighter limits), GONDOR continues beyond this point and often succeeds. Memory cleanups performed by GONDOR also sometimes reduce runtime by discarding unpromising regions. The trade-off is that when GONDOR fails, it typically consumes a larger portion of the time budget, whereas GBFS often fails earlier. For detailed results see Appendix.


\begin{table}[!t]
    \centering
    \footnotesize
    \setlength{\tabcolsep}{1.5pt}
    \begin{tabular}{@{}cl|rrrr@{}}
    \toprule
    & \textbf{Metric} & {\scriptsize GBFS} & {\scriptsize GBFS$_{B\!F}$} & {\scriptsize GONDOR} & {\scriptsize GONDOR$_{B\!F}$} \\
    \midrule
    \multirow{2}{*}{\rotatebox{90}{ \tiny 8 GB}}
    &{\scriptsize Max Coverage (400)}   & 302  & 302  & 307  & 308  \\
    &{\scriptsize Sum Coverage (4000)}  & 2339 & 2377 & 2385 & 2406 \\
    \midrule
    \multirow{2}{*}{\rotatebox{90}{\tiny 512 MB}}
    &{\scriptsize Max Coverage (400)}   & 285  & 290  & 296    & 298  \\
    &{\scriptsize Sum Coverage (4000)}  & 2000 & 2060 & 2099   & 2196 \\
    \bottomrule
    \end{tabular}
    \caption{\label{tab:coverage}\textit{Max} sums the coverage of the \textbf{best} heuristic in each domain, while \textit{Sum} sums the coverage of \textbf{all} heuristics.}
\end{table}

\begin{figure}[!t]
    \includegraphics[width=0.48\textwidth, height=0.28\textheight, keepaspectratio]{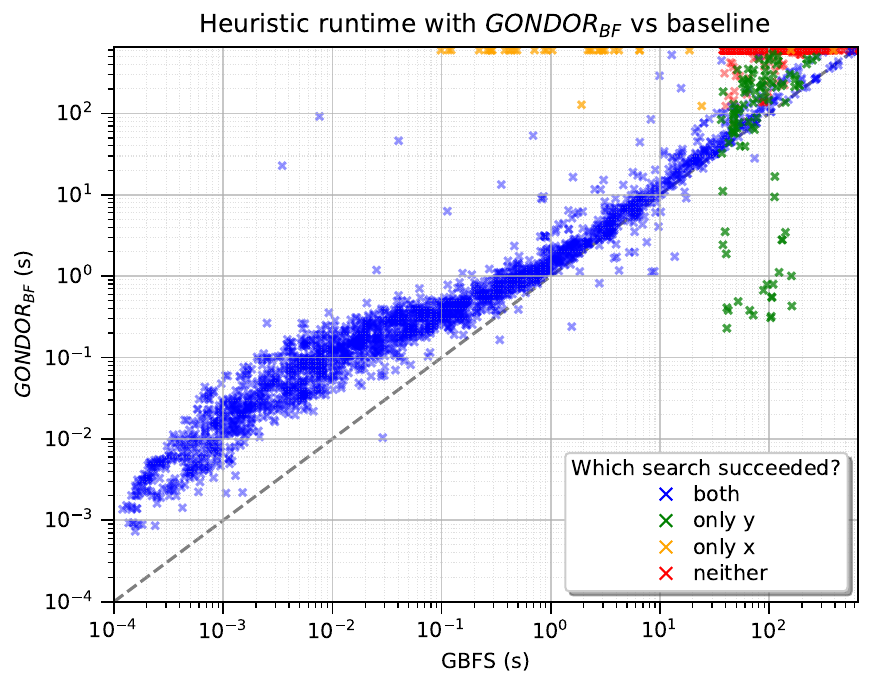}
    \includegraphics[width=0.48\textwidth, height=0.28\textheight, keepaspectratio]{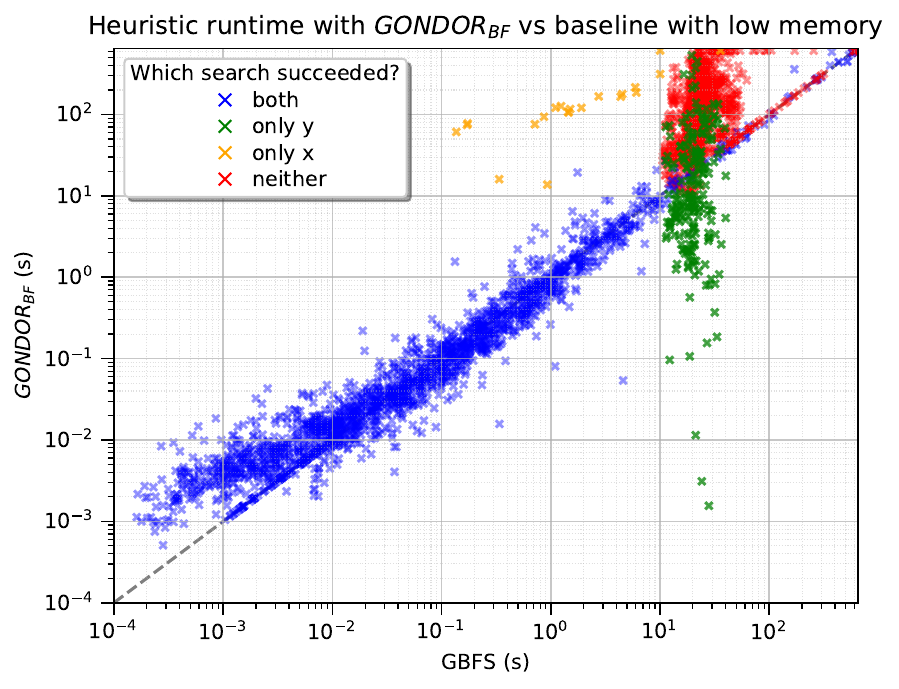}
    \caption{\label{fig:runtime} Runtime comparison of each heuristic and instance for GONDOR$_{B\!F}$ versus the GBFS baseline. Each point corresponds to one run and is colored by which method solved the instance. Points below the diagonal indicate faster performance of GONDOR$_{B\!F}$; points above indicate faster GBFS. Green points denote instances solved by GONDOR$_{B\!F}$, and orange points those solved by GBFS.}
\end{figure}

\section{Conclusion and future work}

We present a new algorithm, GONDOR, which extends GBFS for memory efficient satisficing search, and evaluate a simple variation on numeric planning problems. The experimental results show that even with our naive hyperparameter selection GONDOR yields consistent benefits to coverage which is further improved by using a Bloom Filter as \textsc{closed} list, and the benefit is $\sim3\times$ larger for low-memory settings, at the cost of \textit{failing slower} when it does fail.

Future work should consider better informed \textit{IsOutpost} functions. We find particularly promising the idea of increasing outpost probability $p$ according to the heuristic value of the state relative to $h(s_0)$ and the best achieved so far. Another variation could, instead of determining outpostiness already when generated, only promote nodes to outposts during cleanup.



\bibliography{aaai2026}


\end{document}


\maketitle
\section{algorithms}
We reproduce here pseudocode for all algorithms we describe.

\FloatBarrier

\subsection{GBFS-unlimited}
We reproduce here the GBFS-Unlimited algorithm.

\begin{algorithm}[!ht]
\caption{GBFS-Unlimited}
\label{alg:gbfs-unlimited}
\begin{algorithmic}[1]
\State \textbf{Input:} initial state $s_0$, goal condition $G$, heuristic $h$, successor function $\textsc{Succ}$, memory limit $L$
\State \textbf{Output:} $(status, n)$
\State create root node $r$ for $s_0$
\State $r.parent \leftarrow \bot$; $r.action \leftarrow \bot$; $r.time \leftarrow 0$
\State $\mathit{OPEN} \leftarrow \{r\}$; $\mathit{CLOSED} \leftarrow \{r\}$; $t \leftarrow 1$
\While{$\mathit{OPEN} \neq \emptyset$}
    \State $n \leftarrow$ pop best node from $\mathit{OPEN}$
    \If{$n.state \models G$}
        \State \textbf{return} $(\textsc{success}, n)$
    \EndIf
    \For{each $(a, s') \in \textsc{Succ}(n.state)$}
        \If{$s' \in \mathit{CLOSED}$}
            \State \textbf{continue}
        \EndIf
        \State create child node $m$ for $s'$
        \State $m.parent \leftarrow n$; $m.action \leftarrow a$; $m.time \leftarrow t$
        \State $t \leftarrow t + 1$
        \State insert $m$ into $\mathit{OPEN}$
        \State insert $m$ into $\mathit{CLOSED}$
    \EndFor
    \While{memory usage exceeds $L$ and $\mathit{OPEN} \neq \emptyset$}
        \State remove from $\mathit{OPEN}$ and from $\mathit{CLOSED}$ the node with smallest $time$
    \EndWhile
\EndWhile
\State \textbf{return} $(\textsc{failure}, \bot)$
\end{algorithmic}
\end{algorithm}

\FloatBarrier
\vspace{7em}

\subsection{GBFS-retrace}
We reproduce here the GBFS-Retrace algorithm.

\begin{algorithm}[!ht]
\caption{GBFS-Retrace}
\label{alg:gbfs-retrace}
\begin{algorithmic}[1]
\State \textbf{Input:} initial state $s_0$, goal condition $G$, heuristic $h$, successor function $\textsc{Succ}$, memory limit $L$
\State \textbf{Output:} plan $\pi$ or failure
\State $\pi \leftarrow [\,]$
\While{$True$}
    \State $(status, n) \leftarrow GBFS-Unlimited(s_0, G, h, SUCC, L)$
    \If{status is FAILURE}
        \State break
    \EndIf
    \While{$n.parent \neq None$}
        \State $push~n.action~into~\pi$
        \State $n \leftarrow n.parent$
    \EndWhile
    \If{$n.state == s_0$}
        \State \Comment{we reverse plan since it was pushed end first}
        \State \textbf{return} $REVERSE(\pi)$
    \EndIf

    \State $G \leftarrow lambda~s\in S: s == n$
\EndWhile
\State \textbf{return} $\textsc{failure}$
\end{algorithmic}
\end{algorithm}

\FloatBarrier
\newpage
\subsection{GONDOR}
The algorithm is already in the paper, but it's repeated here for proximity.

\begin{algorithm}[!ht]
\caption{GONDOR}
\label{alg:gondor in appendix}
\begin{algorithmic}[1]
\State \textbf{Input:} initial state $s_0$, goal test $S_{\ast}$, heuristic $h$, successor function $\textsc{Succ}$, memory limit $L$
\State \textbf{Input:} outpost policy $\textsc{IsOutpost}$
\State \textbf{Output:} plan $\pi$ or failure
\State $r \leftarrow \textsc{NewNode}(s_0,\bot)$
\State $r.po \leftarrow \bot$; $r.is\_outpost \leftarrow True$ // po is parent\_outpost
\State $\textsc{open} \leftarrow \{r\}$; $\textsc{closed} \leftarrow \{s_0\}$
\While {$\textsc{open} \neq \emptyset$}
    \State $n \leftarrow$ pop node from $\textsc{open}$ according to $h$
    \If{$n.state \in S_{\ast}$}
        \State \textbf{return} \textsc{ReconstructViaBeacons}$(n, h)$
    \EndIf
        \For{each $s' \in \textsc{Succ}(n.state)$}
        \If{$s' \in \textsc{closed}$}
            \State \textbf{continue}
        \EndIf
        \State $m \leftarrow \textsc{NewNode}(s', n)$
        \State $m.po \leftarrow n~~\mathit{\textbf{if}}~~n.is\_outpost~~\mathit{\textbf{else}}~~n.po$
        \State{$m.is\_outpost \leftarrow \textsc{IsOutpost}(m)$}
        \State $\textsc{open} \leftarrow \textsc{open} \cup \{m\}$
        \State $\textsc{closed} \leftarrow \textsc{closed} \cup \{m.state\} $
    \EndFor

    \If{memory usage exceeds $L$}
        \State $\textsc{open} \leftarrow \textsc{RemoveNonOutposts}(\textsc{open})$
        \State $\textsc{closed} \leftarrow \textsc{open}.state$
    \EndIf
\EndWhile
\State \textbf{return} failure
\end{algorithmic}
\end{algorithm}

\FloatBarrier

\subsection{ReconstructViaBeacons}
We reproduce here the ReconstructViaBeacons algorithm.

\vspace{2em}

\begin{algorithm}[!t]
\caption{\textsc{ReconstructViaBeacons}}
\label{alg:reconstruct}
\begin{algorithmic}[1]
\State \textbf{Input:} goal node $n$, heuristic $h$
\State \textbf{Output:} complete plan $\pi$
\State $\pi_{\mathit{suffix}} \leftarrow [\,]$; $\mathcal{B} \leftarrow [\,]$
\Comment{store the suffix since last clear}
\While{$n.parent \neq \bot$}
    \State push $n.action$ to $\pi_{\mathit{suffix}}$
    \State $n \leftarrow n.parent$
\EndWhile
\Comment{light the beacons}
\State $\pi_{\mathit{suffix}} \leftarrow REVERSE(\pi_{\mathit{suffix}})$
\While{$n \neq \bot$}
    \State push $n.state$ to $\mathcal{B}$
    \State $n \leftarrow n.parent\_outpost$
\EndWhile

\Comment{follow the beacons. NOTE: they were added in reverse order}
\For{$i = 1$ to $|\mathcal{B}| - 1$}
    \State $G \leftarrow lambda~s\in S: s == \mathcal{B}[i]$
    \State $\pi_i \leftarrow \textsc{GBFS}(s_0=\mathcal{B}[i+1], G, h)$
    \State prepend $\pi_i$ to $\pi_{\mathit{suffix}}$
\EndFor
\State \textbf{return} $\pi_{\mathit{suffix}}$
\end{algorithmic}
\end{algorithm}

\section{definitions}

\begin{definition}[Weak Landmark Set]
Let $\mathcal{S} = \langle S, s_I, S^\star, \mathrm{succ} \rangle$ be a state space. A set $W \subseteq S$ is a
\emph{weak landmark set} for $\mathcal{S}$ if there exists an $s_I$-plan
$\pi = \langle s_0, s_1, \ldots, s_n \rangle$ such that
$W \subseteq \{s_0, s_1, \ldots, s_n\}$.
\end{definition}

\noindent This contrasts with the classical notion of a strong landmark, which is an action or fact that must appear in \emph{every} $s_I$-plan. A weak landmark set relaxes
this requirement to existence: suffices that all its elements are collectively visited by at least one plan.

\begin{definition}[Beacon Sequence]
A \emph{beacon sequence} for $\mathcal{S}$ is an ordered tuple
$\mathcal{B} = (b_0, b_1, \ldots, b_k)$ of states satisfying:
\begin{enumerate}
    \item $\{b_0, b_1, \ldots, b_k\}$ is a weak landmark set for $\mathcal{S}$, and
    \item there exists an $s_I$-plan $\langle s_0, s_1, \ldots, s_n \rangle$ in which the
    beacons appear in order: there exist indices
    $0 \leq i_0 < i_1 < \cdots < i_k \leq n$ such that
    $s_{i_j} = b_j$ for all $j \in \{0, \ldots, k\}$.
\end{enumerate}
\end{definition}

\noindent Intuitively, a beacon sequence is a weak landmark set whose order we know along that plan. Every weak landmark set has at least one Beacon sequence, which can be constructed by indexing the states along $\pi$. In GONDOR, phase~1 constructs a valid $\pi$, for which the outposts (or in fact, any subset of states indexed along it) are a valid beacon sequence.

\paragraph{Analogy to Weak and Strong Stubborn Sets.}
The weak/strong relation introduced above for landmark sets mirrors an analogous relation in the stubborn-set literature~\cite{valmari1989stubborn,sievers2021weak}. Stubborn sets are a partial-order-reduction technique that prunes the successors considered at each search state. \emph{Strong} stubborn sets (SSS) impose the requirement that every optimal plan in the original space corresponds to a plan of equal cost in the pruned space---a universal requirement over the entire solution set.

\emph{Weak} stubborn sets (WSS), as defined by Valmari in the original formulation and recently revisited for classical planning by \citet{sievers2021weak}, relax this to an existential guarantee: the pruned space is only required to preserve \emph{at least one} optimal plan, permitting more aggressive pruning.

Strong landmarks and strong stubborn sets both satisfy a \emph{universal} condition, whereas the weak landmark sets and weak stubborn sets both satisfy an \emph{existential} condition. Every strong stubborn set is also a weak one, and every set of (strong) state landmarks is also a weak landmark set. GONDOR's beacons are not landmarks, but they are a weak landmark set which suffices for path reconstruction.

\section{coverage results}
\label{coverage results full section}
Below are the per-domain coverage results

\begin{table}[!ht]
\centering
\begin{tabular}{@{}l|rr|rr@{}}
 & \multicolumn{2}{|c}{GBFS} & \multicolumn{2}{|c}{GONDOR}\\
Domain & no-BF & BF & no-BF & BF \\
\midrule
Block Grouping (20)  & 19 & 19 & 19 & 19 \\
Counters       (20)  & 10 & 10 & \textbf{14} & \textbf{14} \\
Delivery       (20)  & 18 & 18 & \textbf{19} & 18 \\
Drone          (20)  & 19 & 19 & 19 & \textbf{20} \\
Expedition     (20)  & 4  &  4 &  4 &  4 \\
Farming        (20)  & 20 & 20 & 20 & 20 \\
FO-Counters    (20)  & 7  &  7 &  7 &  7 \\
FO-Farming     (20)  & 20 & 20 & 20 & 20 \\
FO-Sailing     (20)  & 20 & 20 & 20 & 20 \\
Hydropower     (20)  & 20 & 20 & 20 & 20 \\
Market Trader  (20)  & 20 & 20 & 20 & 20 \\
Pacman         (20)  & 16 & 16 & 16 & 16 \\
Pathways       (20)  & 1  & 1  & 1  & 2  \\
Plant Watering (20)  & 20 & 20 & 20 & 20 \\
Rover          (20)  & 4  & 4  & 4  & 4  \\
Sailing        (20)  & 20 & 20 & 20 & 20 \\
Settlers       (20)  & 4  & 4  & 4  & 4  \\
TPP            (20)  & 20 & 20 & 20 & 20 \\
Twin Prime     (20)  & 20 & 20 & 20 & 20 \\
Zenotravel     (20)  & 20 & 20 & 20 & 20 \\
\midrule
$\sum$         (400) & 302 & 302 & 307 & \textbf{308} \\
\end{tabular}
\caption{\label{tab:max coverage full} Max coverage per-domain of each search algorithm on the standard (8 GB RAM) setting}
\end{table}

\begin{table}[!t]
\centering
\begin{tabular}{@{}l|rr|rr@{}}
 & \multicolumn{2}{|c}{GBFS} & \multicolumn{2}{|c}{GONDOR}\\
Domain & no-BF & BF & no-BF & BF \\
\midrule
Block Grouping (200)  & 163  & 158  & 161  & 155 \\
Counters       (200)  & 78   & 78   & 97   & 98 \\
Delivery       (200)  & 143  & 143  & 144  & 144 \\
Drone          (200)  & 149  & 154  & 149  & 157 \\
Expedition     (200)  & 25   & 28   & 22   & 25 \\
Farming        (200)  & 178  & 178  & 178  & 178 \\
FO-Counters    (200)  & 53   & 53   & 54   & 55 \\
FO-Farming     (200)  & 180  & 180  & 180  & 180 \\
FO-Sailing     (200)  & 145  & 149  & 157  & 155 \\
Hydropower     (200)  & 123  & 140  & 124  & 128 \\
Market Trader  (200)  & 127  & 127  & 133  & 133 \\
Pacman         (200)  & 137  & 140  & 143  & 145 \\
Pathways       (200)  & 8    & 8    & 8    & 9  \\
Plant Watering (200)  & 177  & 187  & 184  & 191 \\
Rover          (200)  & 20   & 20   & 20   & 20 \\
Sailing        (200)  & 147  & 149  & 151  & 152 \\
Settlers       (200)  & 23   & 23   & 24   & 23 \\
TPP            (200)  & 133  & 132  & 134  & 133 \\
Twin Prime     (200)  & 172  & 172  & 170  & 170 \\
Zenotravel     (200)  & 158  & 158  & 152  & 155 \\
\midrule
$\sum$         (4000) & 2339 & 2377 & 2385 & 2406 \\
\end{tabular}
\caption{\label{tab:sum coverage full} Sum coverage per-domain of each search algorithm on the standard (8 GB RAM) setting}
\end{table}

\begin{table}[!t]
\centering
\begin{tabular}{@{}l|rr|rr@{}}
 & \multicolumn{2}{|c}{GBFS} & \multicolumn{2}{|c}{GONDOR}\\
Domain & no-BF & BF & no-BF & BF \\
\midrule
Block Grouping (20)  & 19  & 19 & 19 & 19 \\
Counters       (20)  & 8   & 9  & 10 & 10 \\
Delivery       (20)  & 18  & 18 & 18 & 18 \\
Drone          (20)  & 16  & 16 & 17 & 17 \\
Expedition     (20)  & 3   & 3  & 3  & 3  \\
Farming        (20)  & 20  & 20 & 20 & 20 \\
FO-Counters    (20)  & 6   & 6  & 6  & 6  \\
FO-Farming     (20)  & 20  & 20 & 20 & 20 \\
FO-Sailing     (20)  & 20  & 20 & 20 & 20 \\
Hydropower     (20)  & 20  & 20 & 20 & 20 \\
Market Trader  (20)  & 20  & 20 & 20 & 20 \\
Pacman         (20)  & 16  & 16 & 16 & 16 \\
Pathways       (20)  & 1   & 1  & 1  & 1  \\
Plant Watering (20)  & 17  & 20 & 20 & 20 \\
Rover          (20)  & 4   & 4  & 4  & 4  \\
Sailing        (20)  & 20  & 20 & 20 & 20 \\
Settlers       (20)  & 4   & 4  & 4  & 4  \\
TPP            (20)  & 13  & 14 & 18 & 20 \\
Twin Prime     (20)  & 20  & 20 & 20 & 20 \\
Zenotravel     (20)  & 20  & 20 & 20 & 20 \\
\midrule
$\sum$         (400) & 285 & 290 & 296 & \textbf{298} \\
\end{tabular}
\caption{\label{tab:max coverage full lowmem} Max coverage per-domain of each search algorithm on the low memory (512 MB) setting.}
\end{table}

\begin{table}[!t]
\centering
\begin{tabular}{@{}l|rr|rr@{}}
 & \multicolumn{2}{|c}{GBFS} & \multicolumn{2}{|c}{GONDOR}\\
Domain & no-BF & BF & no-BF & BF \\
\midrule
Block Grouping (200)  & 148  & 150 & 140 & 156  \\
Counters       (200)  & 68   & 73  & 73  & 78   \\
Delivery       (200)  & 143  & 143 & 146 & 143  \\
Drone          (200)  & 139  & 139 & 145 & 145  \\
Expedition     (200)  & 17   & 20  & 18  & 19   \\
Farming        (200)  & 169  & 169 & 170 & 171  \\
FO-Counters    (200)  & 45   & 45  & 51  & 51   \\
FO-Farming     (200)  & 171  & 171 & 176 & 180  \\
FO-Sailing     (200)  & 127  & 128 & 133 & 136  \\
Hydropower     (200)  & 108  & 112 & 108 & 112  \\
Market Trader  (200)  & 104  & 104 & 117 & 122  \\
Pacman         (200)  & 116  & 122 & 125 & 129  \\
Pathways       (200)  & 7    & 8   & 8   & 8    \\
Plant Watering (200)  & 89   & 118 & 106 & 154  \\
Rover          (200)  & 20   & 20  & 20  & 20   \\
Sailing        (200)  & 115  & 121 & 125 & 127  \\
Settlers       (200)  & 14   & 13  & 21  & 23   \\
TPP            (200)  & 79   & 81  & 91  & 103  \\
Twin Prime     (200)  & 172  & 172 & 171 & 170  \\
Zenotravel     (200)  & 149  & 151 & 155 & 149  \\
\midrule
$\sum$         (4000) & 2000 & 2060 & 2099 & 2196 \\
\end{tabular}
\caption{\label{tab:sum coverage full lowmem} Sum coverage per-domain of each search algorithm on the low memory (512 MB) setting.}
\end{table}

\FloatBarrier

\section{scatterplots}
Below are the scatterplots of run time, expanded nodes, generated nodes, and solution length of $GONDOR_{BF}$ as well as $GONDOR$ and $GBFS_{BF}$ against the baseline. We note that expanded and generated notes are not directly comparable between GONDOR and non-GONDOR methods since GONDOR may expand or generate the same node multiple times. We nevertheless present them here for completion's sake.

Some artifacts appear in expanded nodes = 1000, 2000, 3000. These are logging artifacts, as for failures we only log every 1000 expansions.

\FloatBarrier

\subsection{run times}

\begin{figure}[!ht]
    \includegraphics[width=0.48\textwidth, height=0.28\textheight, keepaspectratio]{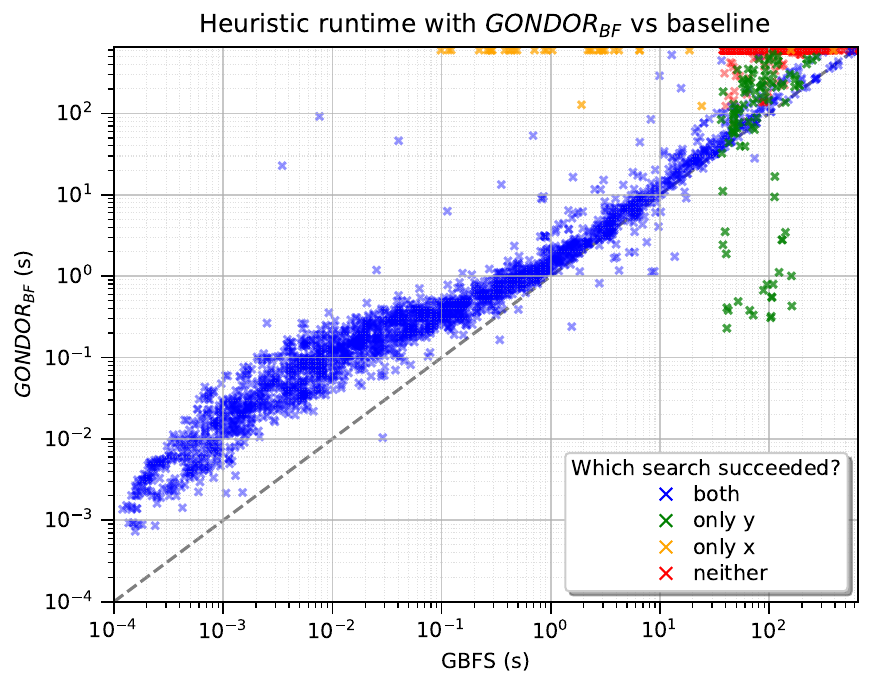}
    \includegraphics[width=0.48\textwidth, height=0.28\textheight, keepaspectratio]{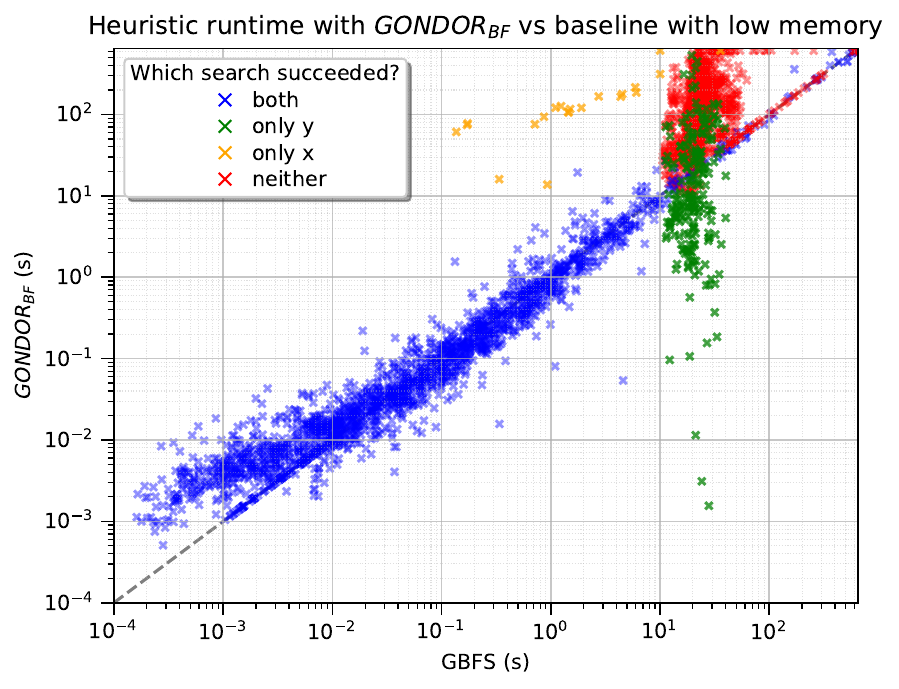}
    \caption{Runtime comparison of each heuristic on each instance of $GONDOR_{BF}$ vs the $GBFS$ baseline. Each X represents one search with each method, and is colored by which search algorithm found a solution successfully. $GONDOR_{BF}$ wins for points below the diagonal Points below the diagonal and points in green, and loses above the diagonal and for points in orange. The same experiment is repeated with 8GB RAM (top) and 512 MB (bottom)}
\end{figure}

\begin{figure}[!ht]
    \includegraphics[width=0.48\textwidth, height=0.28\textheight, keepaspectratio]{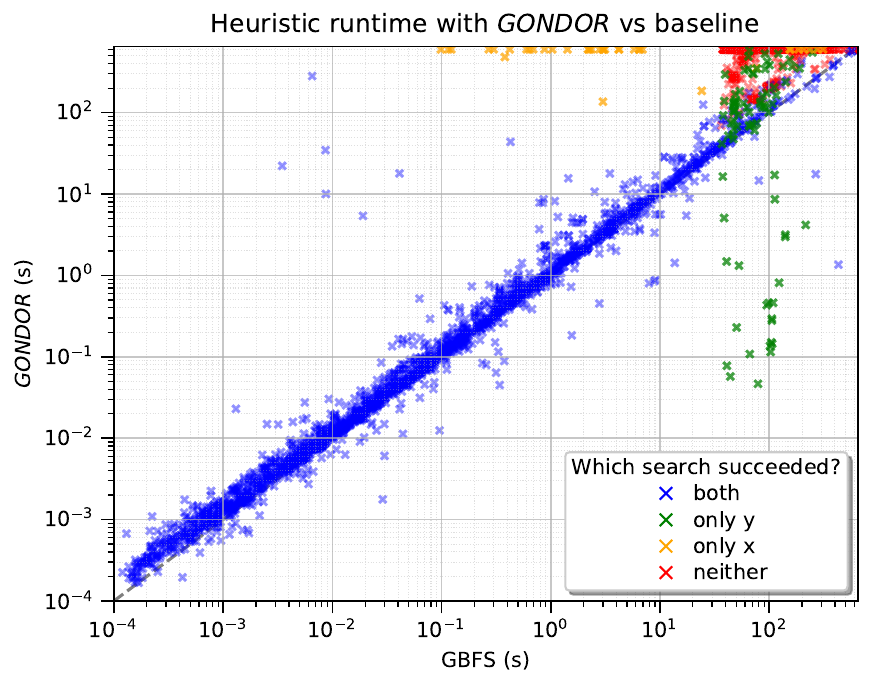}
    \includegraphics[width=0.48\textwidth, height=0.28\textheight, keepaspectratio]{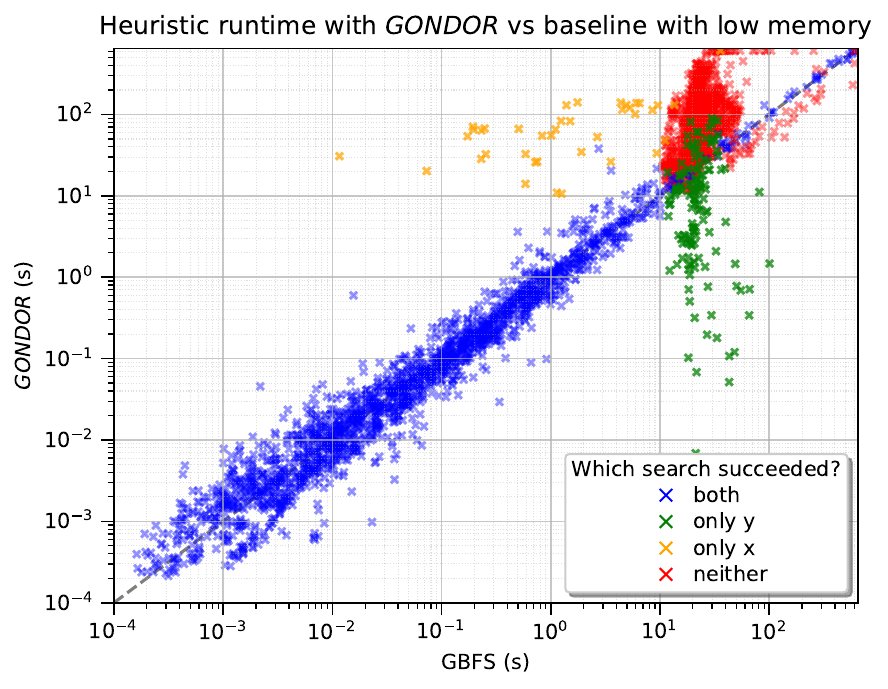}
    \caption{Runtime comparison of each heuristic on each instance of $GONDOR$ vs the $GBFS$ baseline. Each X represents one search with each method, and is colored by which search algorithm found a solution successfully. $GONDOR$ wins for points below the diagonal Points below the diagonal and points in green, and loses above the diagonal and for points in orange. The same experiment is repeated with 8GB RAM (top) and 512 MB (bottom)}
\end{figure}

\begin{figure}[!ht]
    \includegraphics[width=0.48\textwidth, height=0.28\textheight, keepaspectratio]{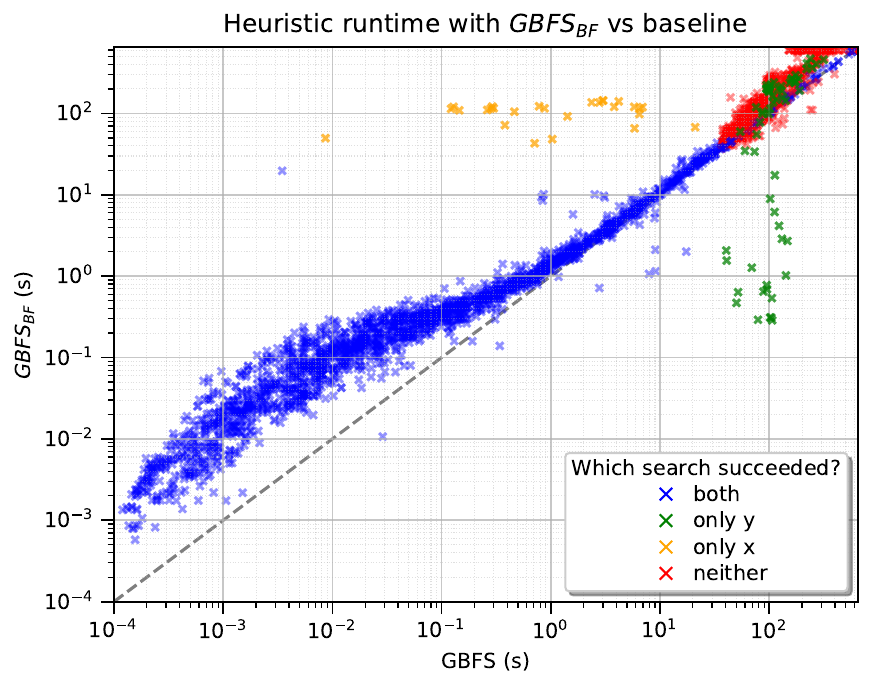}
    \includegraphics[width=0.48\textwidth, height=0.28\textheight, keepaspectratio]{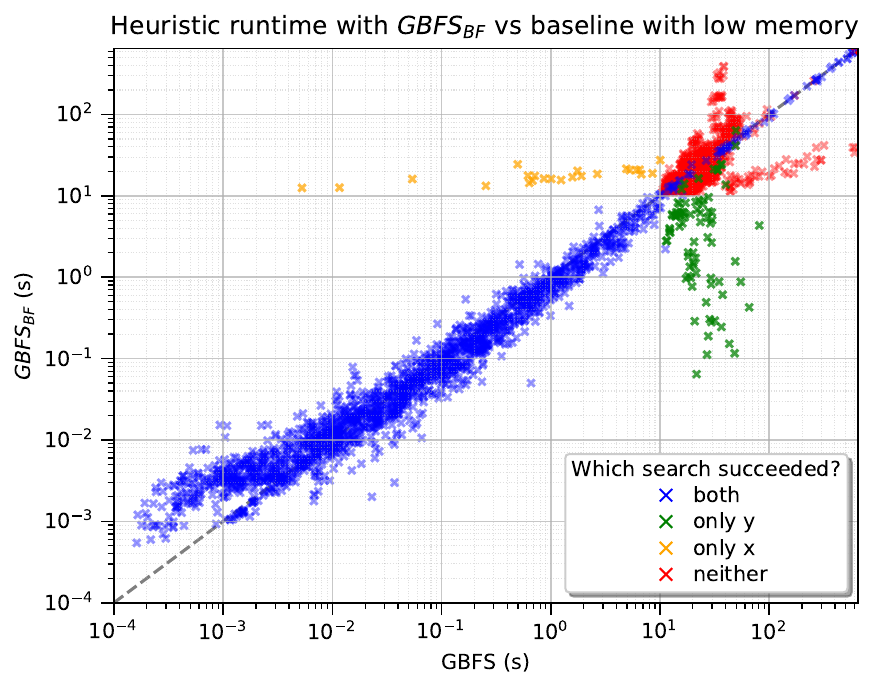}
    \caption{Runtime comparison of each heuristic on each instance of $GBFS_{BF}$ vs the $GBFS$ baseline. Each X represents one search with each method, and is colored by which search algorithm found a solution successfully. $GBFS_{BF}$ wins for points below the diagonal Points below the diagonal and points in green, and loses above the diagonal and for points in orange. The same experiment is repeated with 8GB RAM (top) and 512 MB (bottom)}
\end{figure}

\FloatBarrier
\vspace{30em}

\subsection{expanded nodes}

\begin{figure}[!ht]
    \includegraphics[width=0.48\textwidth, height=0.28\textheight, keepaspectratio]{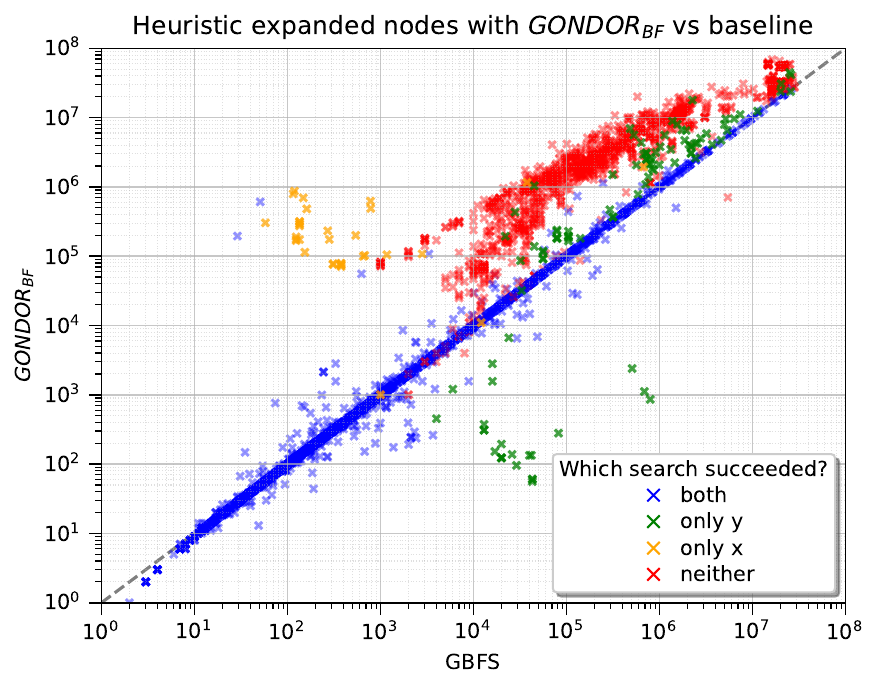}
    \includegraphics[width=0.48\textwidth, height=0.28\textheight, keepaspectratio]{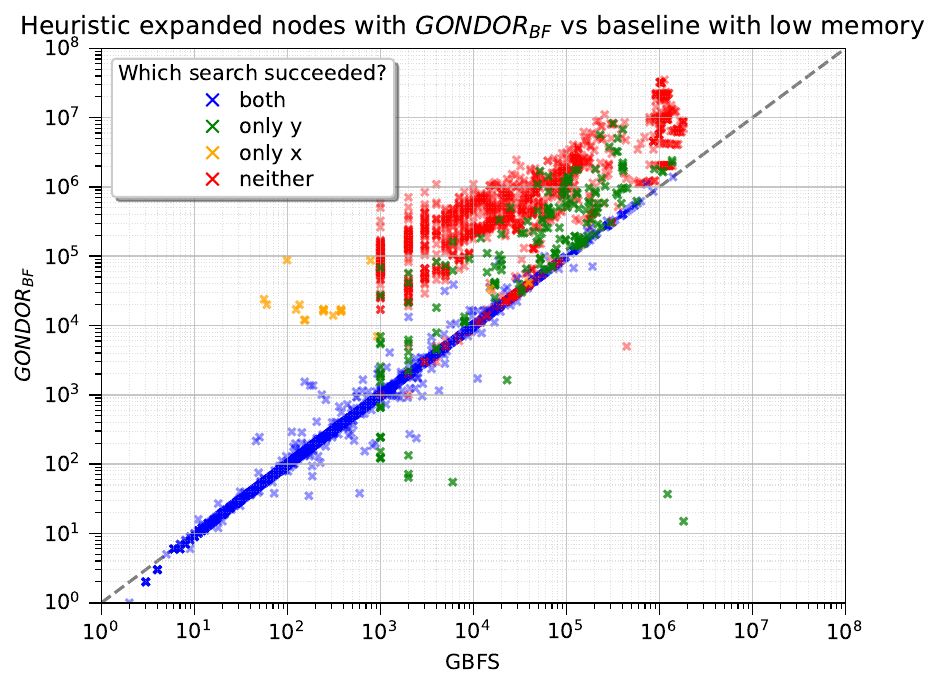}
    \caption{Expanded nodes comparison of each heuristic on each instance of $GONDOR_{BF}$ vs the $GBFS$ baseline. Each X represents one search with each method, and is colored by which search algorithm found a solution successfully. $GONDOR_{BF}$ wins for points below the diagonal Points below the diagonal and points in green, and loses above the diagonal and for points in orange. The same experiment is repeated with 8GB RAM (top) and 512 MB (bottom)}
\end{figure}

\begin{figure}[!ht]
    \includegraphics[width=0.48\textwidth, height=0.28\textheight, keepaspectratio]{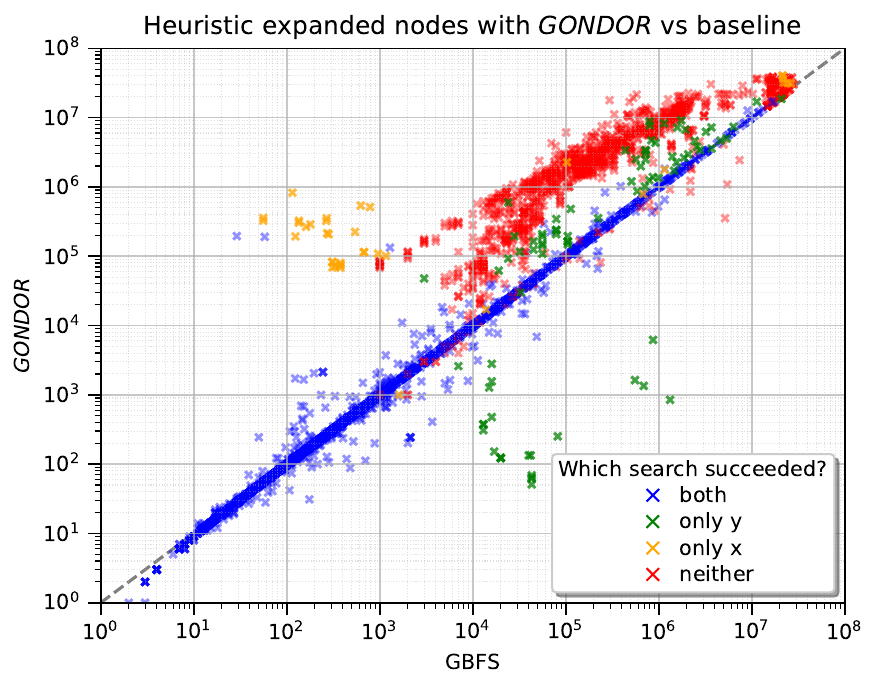}
    \includegraphics[width=0.48\textwidth, height=0.28\textheight, keepaspectratio]{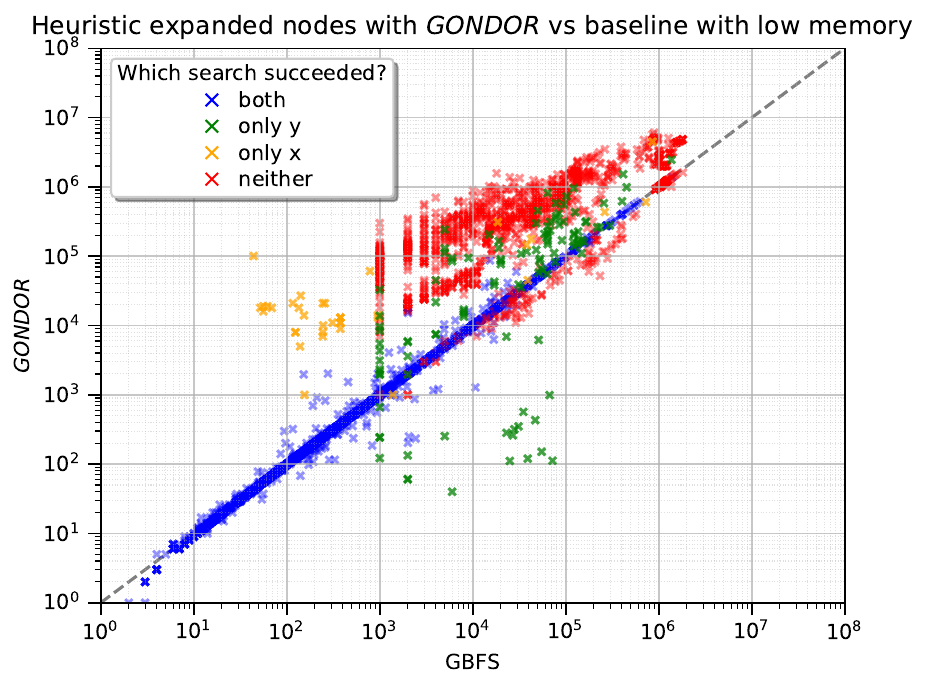}
    \caption{Expanded nodes comparison of each heuristic on each instance of $GONDOR$ vs the $GBFS$ baseline. Each X represents one search with each method, and is colored by which search algorithm found a solution successfully. $GONDOR$ wins for points below the diagonal Points below the diagonal and points in green, and loses above the diagonal and for points in orange. The same experiment is repeated with 8GB RAM (top) and 512 MB (bottom)}
\end{figure}

\begin{figure}[!ht]
    \includegraphics[width=0.48\textwidth, height=0.28\textheight, keepaspectratio]{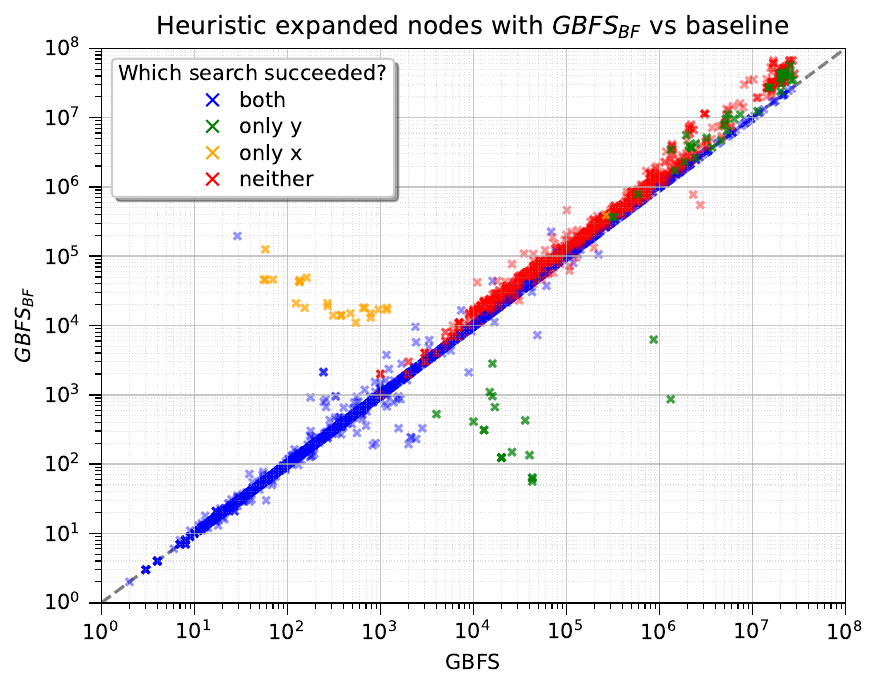}
    \includegraphics[width=0.48\textwidth, height=0.28\textheight, keepaspectratio]{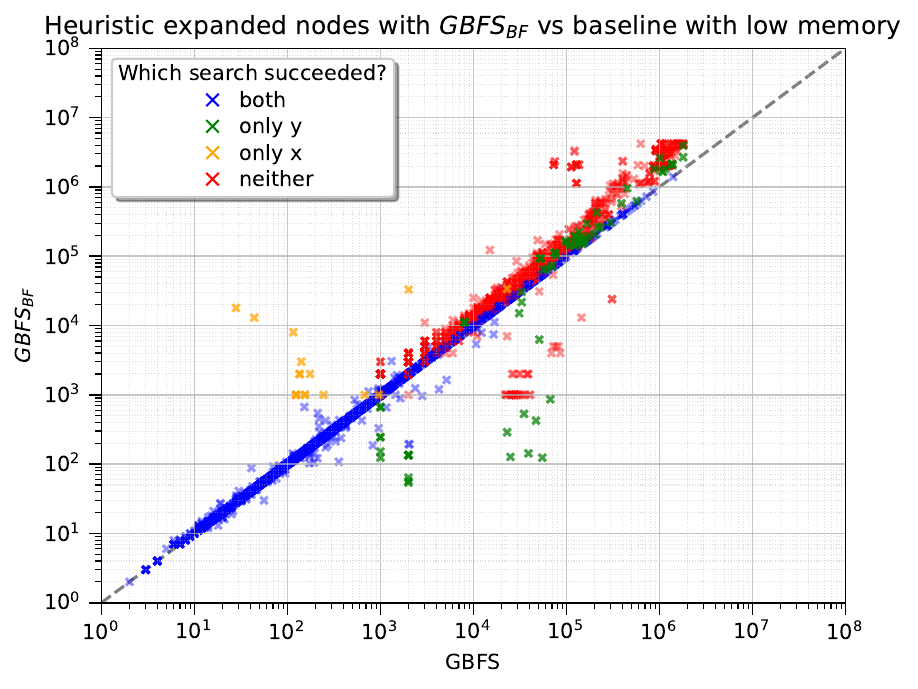}
    \caption{Expanded nodes comparison of each heuristic on each instance of $GBFS_{BF}$ vs the $GBFS$ baseline. Each X represents one search with each method, and is colored by which search algorithm found a solution successfully. $GBFS_{BF}$ wins for points below the diagonal Points below the diagonal and points in green, and loses above the diagonal and for points in orange. The same experiment is repeated with 8GB RAM (top) and 512 MB (bottom)}
\end{figure}

\FloatBarrier

\subsection{generated nodes}

\begin{figure}[!ht]
    \includegraphics[width=0.48\textwidth, height=0.28\textheight, keepaspectratio]{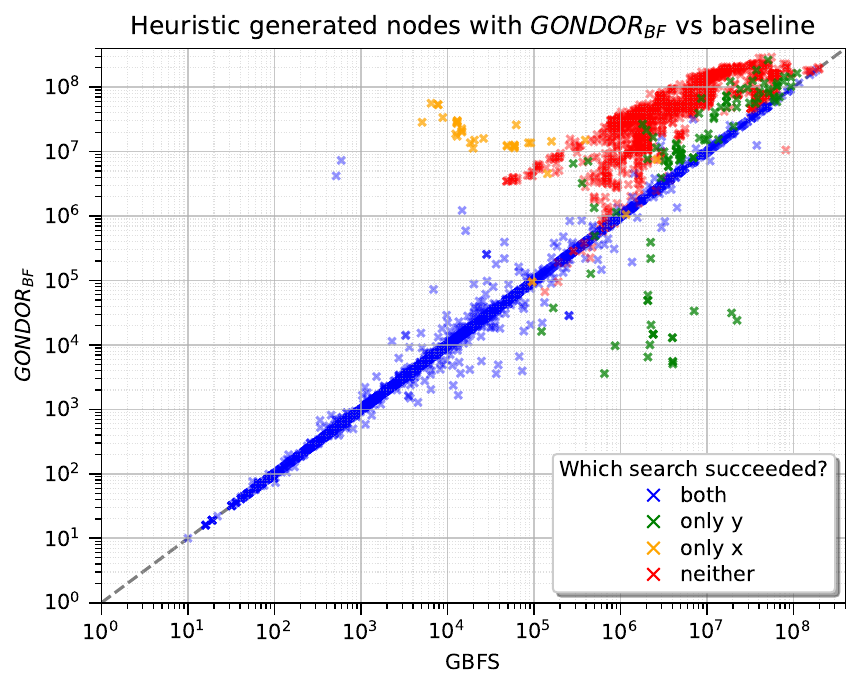}
    \includegraphics[width=0.48\textwidth, height=0.28\textheight, keepaspectratio]{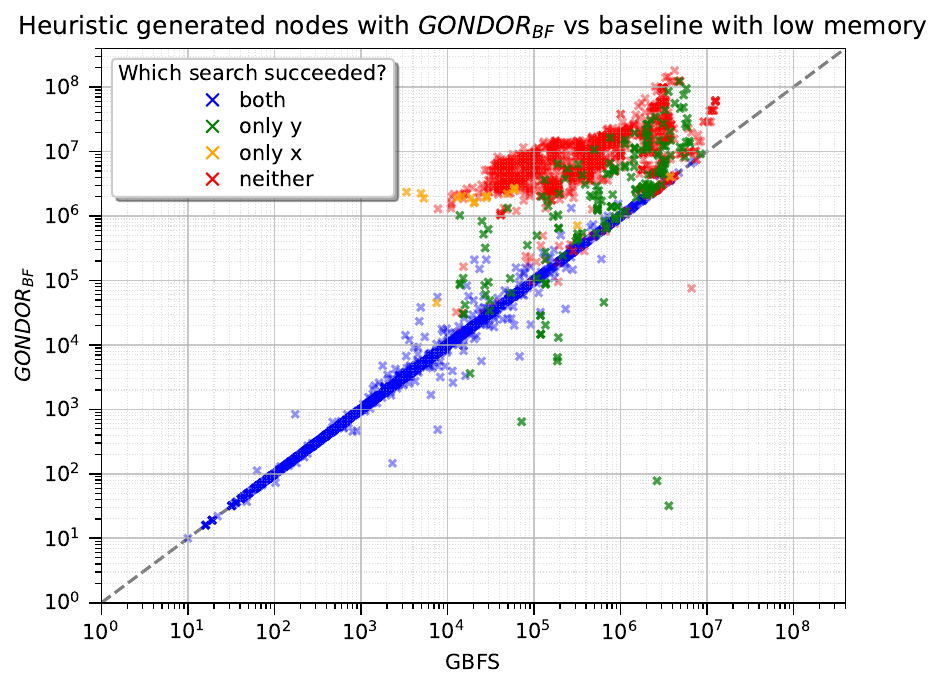}
    \caption{Generated nodes comparison of each heuristic on each instance of $GONDOR_{BF}$ vs the $GBFS$ baseline. Each X represents one search with each method, and is colored by which search algorithm found a solution successfully. $GONDOR_{BF}$ wins for points below the diagonal Points below the diagonal and points in green, and loses above the diagonal and for points in orange. The same experiment is repeated with 8GB RAM (top) and 512 MB (bottom)}
\end{figure}

\begin{figure}[!ht]
    \includegraphics[width=0.48\textwidth, height=0.28\textheight, keepaspectratio]{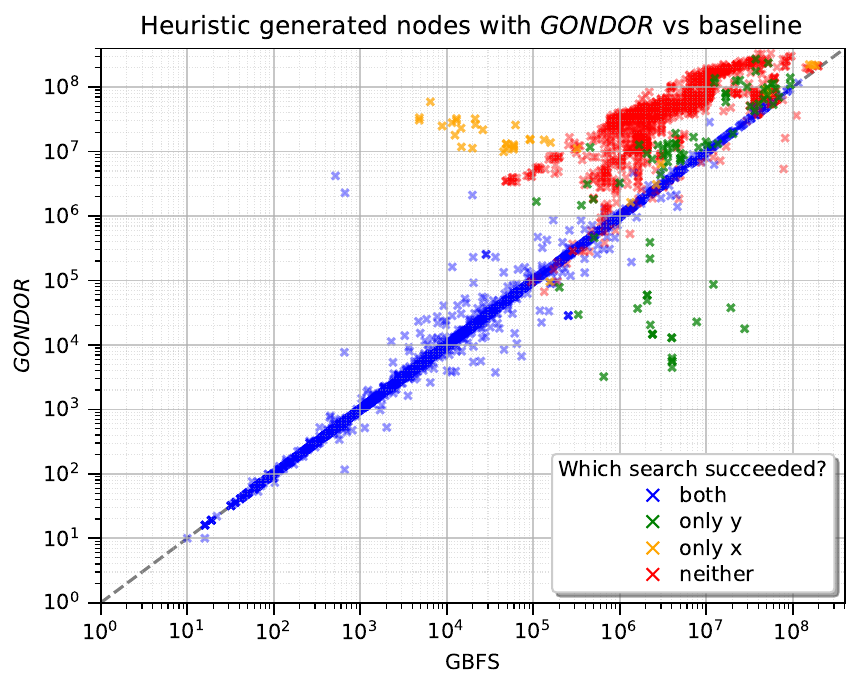}
    \includegraphics[width=0.48\textwidth, height=0.28\textheight, keepaspectratio]{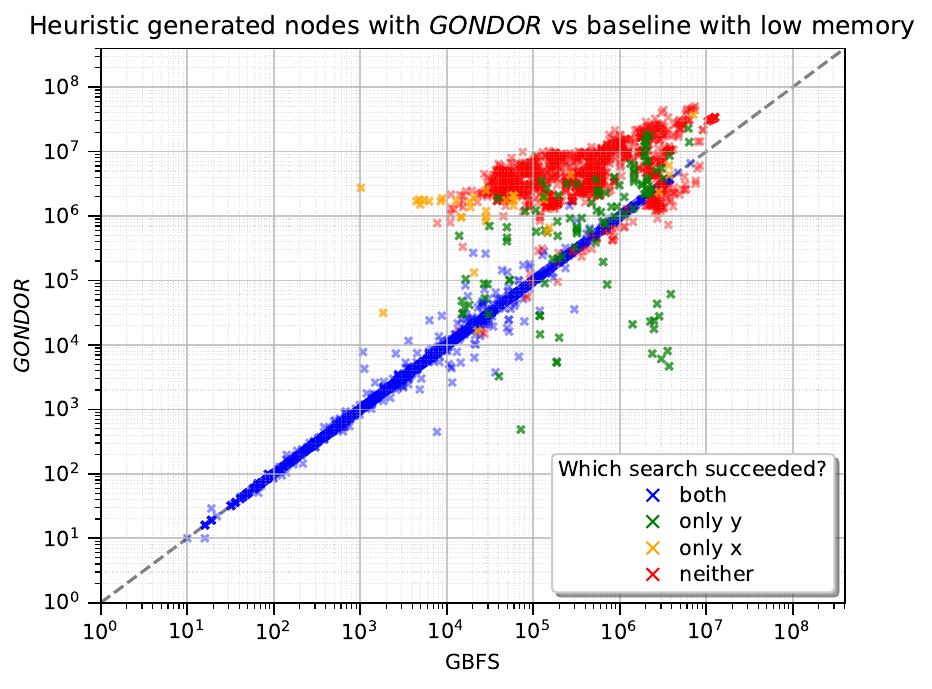}
    \caption{Generated nodes comparison of each heuristic on each instance of $GONDOR$ vs the $GBFS$ baseline. Each X represents one search with each method, and is colored by which search algorithm found a solution successfully. $GONDOR$ wins for points below the diagonal Points below the diagonal and points in green, and loses above the diagonal and for points in orange. The same experiment is repeated with 8GB RAM (top) and 512 MB (bottom)}
\end{figure}

\begin{figure}[!ht]
    \includegraphics[width=0.48\textwidth, height=0.28\textheight, keepaspectratio]{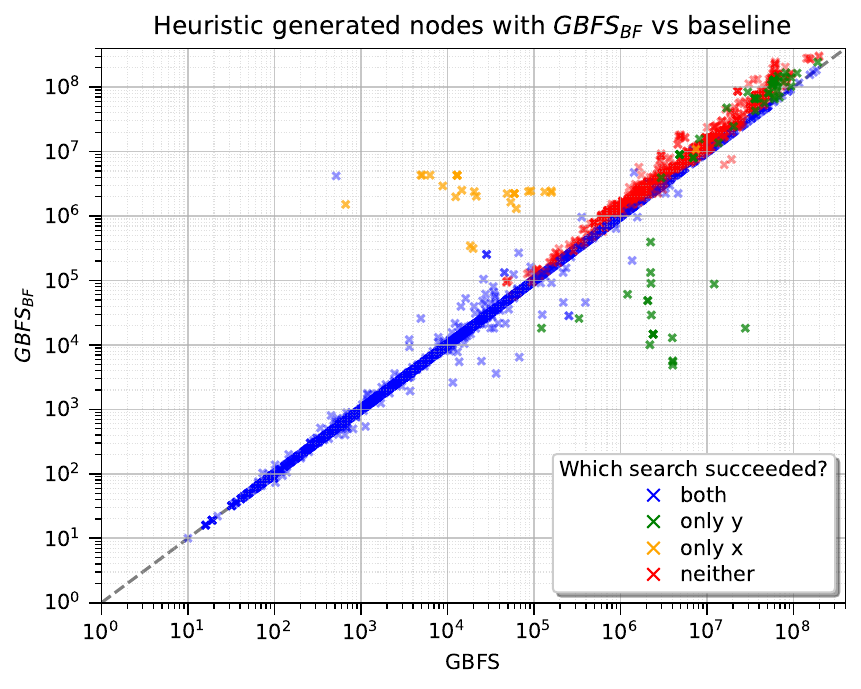}
    \includegraphics[width=0.48\textwidth, height=0.28\textheight, keepaspectratio]{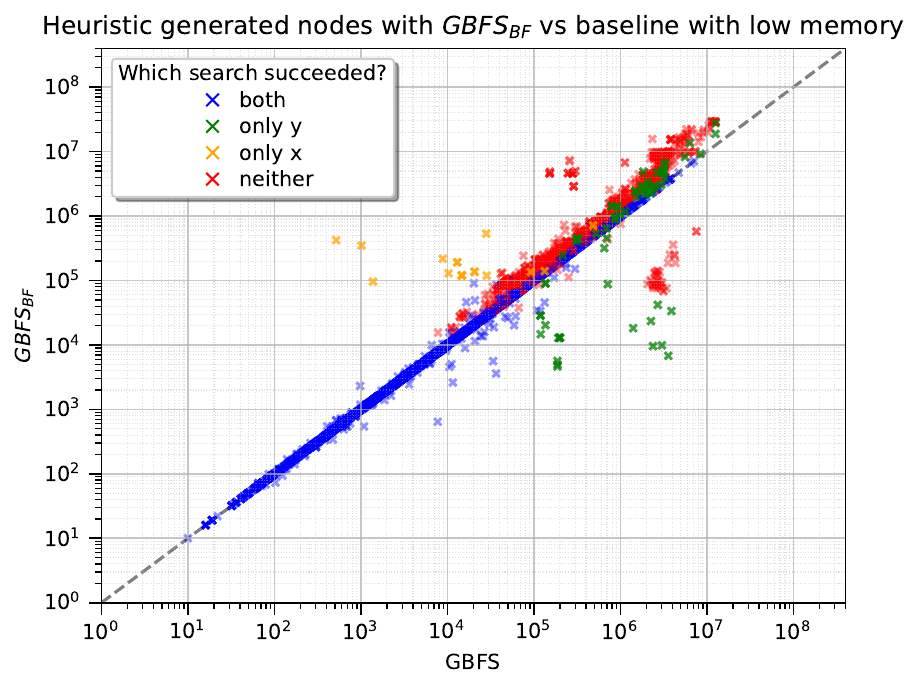}
    \caption{Generated nodes comparison of each heuristic on each instance of $GBFS_{BF}$ vs the $GBFS$ baseline. Each X represents one search with each method, and is colored by which search algorithm found a solution successfully. $GBFS_{BF}$ wins for points below the diagonal Points below the diagonal and points in green, and loses above the diagonal and for points in orange. The same experiment is repeated with 8GB RAM (top) and 512 MB (bottom)}
\end{figure}

\FloatBarrier

\subsection{solution lengths}

\begin{figure}[!ht]
    \includegraphics[width=0.48\textwidth, height=0.28\textheight, keepaspectratio]{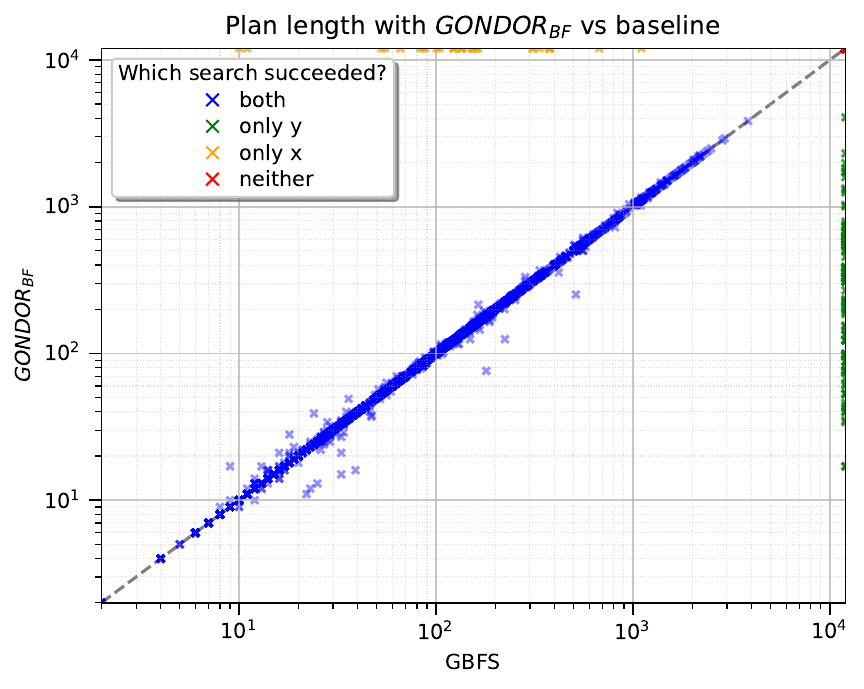}
    \includegraphics[width=0.48\textwidth, height=0.28\textheight, keepaspectratio]{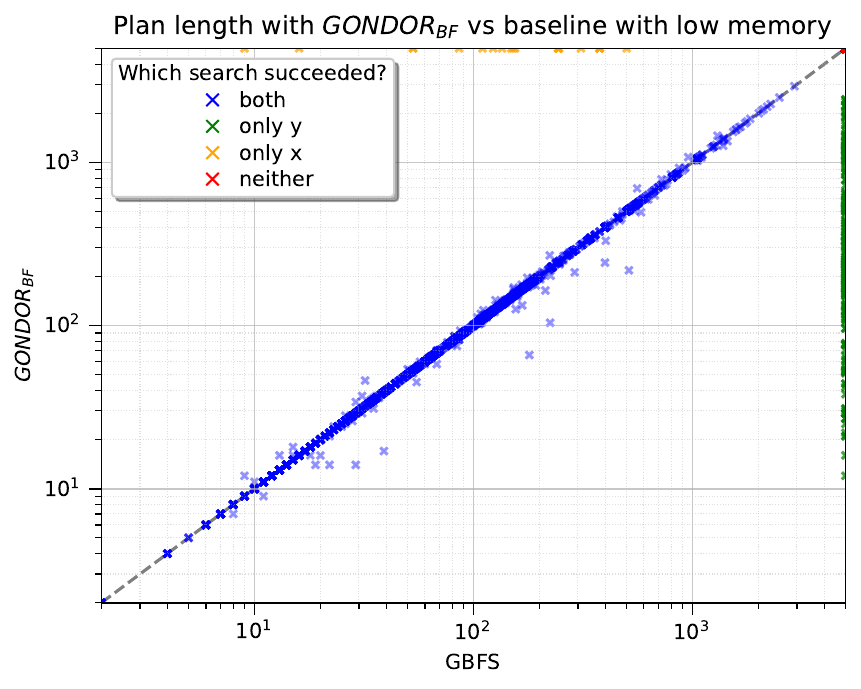}
    \caption{Solution length comparison of each heuristic on each instance of $GONDOR_{BF}$ vs the $GBFS$ baseline. Each X represents one search with each method, and is colored by which search algorithm found a solution successfully, while failed searches are clipped to max. $GONDOR_{BF}$ wins for points below the diagonal Points below the diagonal and points in green, and loses above the diagonal and for points in orange. The same experiment is repeated with 8GB RAM (top) and 512 MB (bottom)}
\end{figure}

\begin{figure}[!ht]
    \includegraphics[width=0.48\textwidth, height=0.28\textheight, keepaspectratio]{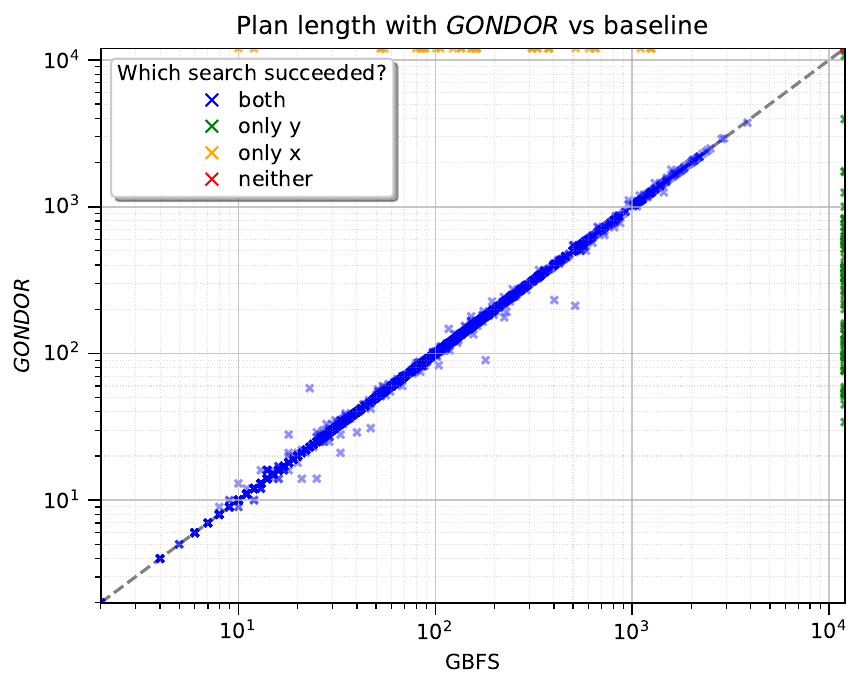}
    \includegraphics[width=0.48\textwidth, height=0.28\textheight, keepaspectratio]{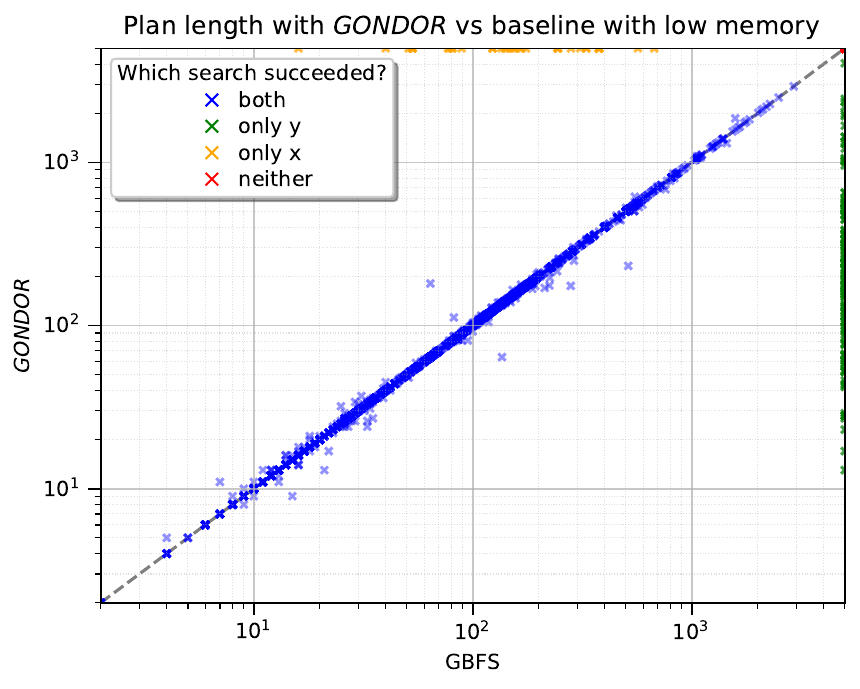}
    \caption{Solution length comparison of each heuristic on each instance of $GONDOR$ vs the $GBFS$ baseline. Each X represents one search with each method, and is colored by which search algorithm found a solution successfully, while failed searches are clipped to max. $GONDOR$ wins for points below the diagonal Points below the diagonal and points in green, and loses above the diagonal and for points in orange. The same experiment is repeated with 8GB RAM (top) and 512 MB (bottom)}
\end{figure}

\begin{figure}[!ht]
    \includegraphics[width=0.48\textwidth, height=0.28\textheight, keepaspectratio]{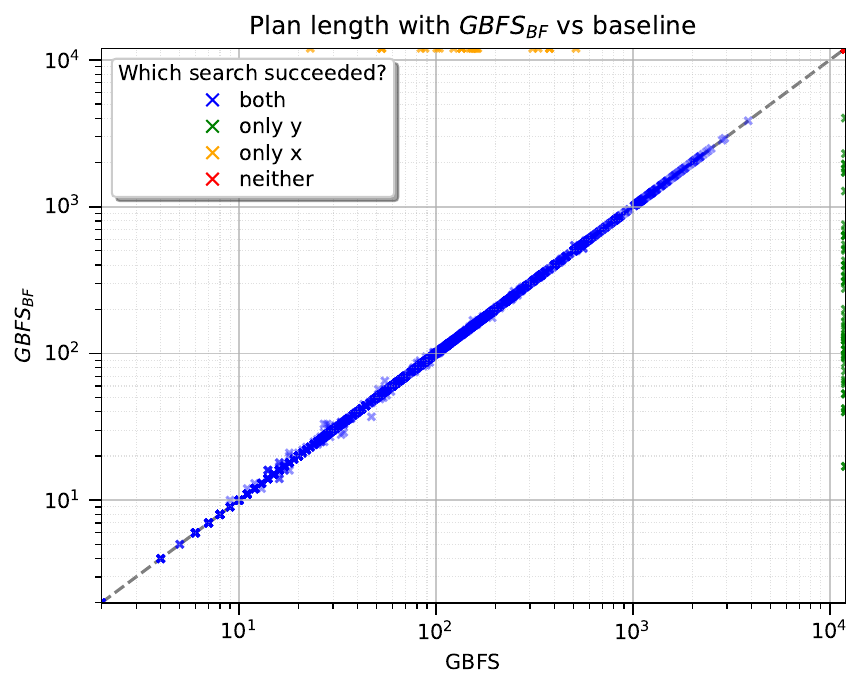}
    \includegraphics[width=0.48\textwidth, height=0.28\textheight, keepaspectratio]{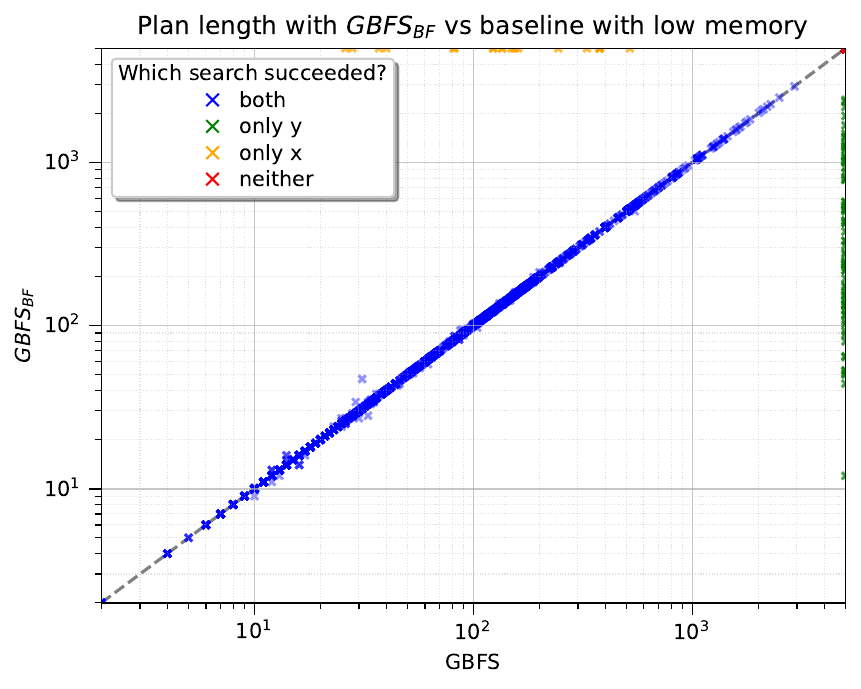}
    \caption{Solution length comparison of each heuristic on each instance of $GBFS_{BF}$ vs the $GBFS$ baseline. Each X represents one search with each method, and is colored by which search algorithm found a solution successfully, while failed searches are clipped to max. $GBFS_{BF}$ wins for points below the diagonal Points below the diagonal and points in green, and loses above the diagonal and for points in orange. The same experiment is repeated with 8GB RAM (top) and 512 MB (bottom)}
\end{figure}

\FloatBarrier

\bibliography{aaai2026}